\documentclass[11pt]{article}

\usepackage{amsmath}
\usepackage{amssymb}
\usepackage{mathtools}
\usepackage{amsthm}
\usepackage[textsize=tiny]{todonotes}
\usepackage[capitalize,noabbrev]{cleveref}
\usepackage{thm-restate}
\usepackage{algorithm,algorithmic}
\usepackage{geometry}
\newtheorem{theorem}{Theorem}[section]
\newtheorem{lemma}{Lemma}[section]

\def\B{{\mathbb{B}}}

\def\I{{\mathbb{I}}}
\def\O{{\mathbb{O}}}
\def\L{{\mathbb{L}}}
\def\E{{\mathbb{E}}}

\def\R{{\mathbb{R}}}

\def\N{{\mathbb{N}}}

\def\ID{{\mathbf{I}}}

\def\A{{\mathcal{A}}}
\def\Ca{{\mathcal{C}}}
\def\M{{\mathcal{M}}}
\def\U{{\mathcal{U}}}

\def\F{{\mathcal{F}}}
\def\G{{\mathcal{G}}}
\def\D{{\mathcal{D}}}
\def\NT{{\mathcal{N}}}

\def\S{{\mathcal{S}}}

\def\Relu{{\hbox{\rm{Relu}}}}

\def\sign{{\hbox{\rm{sign}}}}
\def\diag{{\hbox{\rm{diag}}}}
\def\CE{{\hbox{\scriptsize\rm{CE}}}}

\def\d{{\hbox{\rm{d}}}}

\def\dint{{\displaystyle\int}}

\begin{document}

\title{Robust and Information-theoretically Safe Bias Classifier\\ against Adversarial Attacks\thanks{This work is partially supported by NSFC grant No.11688101 and  NKRDP grant No.2018YFA0306702.}}
\author{Lijia Yu and Xiao-Shan Gao\\
 Academy of Mathematics and Systems Science,  Chinese Academy of Sciences, \\ Beijing 100190,  China\\
 University of  Chinese Academy of Sciences,  Beijing 100049,  China\\
Email: xgao@mmrc.iss.ac.cn}
%\date{\today}
\date{}
\maketitle

\begin{abstract}
\noindent
In this paper, the bias classifier is introduced, that is, the bias part
of a DNN with Relu as the activation function is used as a classifier.
The work is motivated by the fact that
the bias part is a piecewise constant function with zero gradient and hence  cannot be
directly attacked by gradient-based methods to generate adversaries, such as FGSM.
The existence of the bias classifier is proved and an effective training method
for the bias classifier is given.
It is proved that by adding a proper random first-degree part to the bias classifier, an information-theoretically safe classifier against the original-model gradient attack is obtained
in the sense that the attack will generate a totally random attacking direction.
This seems to be the first time that the concept of information-theoretically safe classifier is proposed.
Several attack methods for the bias classifier are proposed
and numerical experiments  are used to show that
the bias classifier is more robust than DNNs with similar size against these attacks in most cases.

\end{abstract}

\section{Introduction}
The deep neural network (DNN) \cite{lecun2015deep} has become the most powerful machine learning method, which has been successfully applied in computer vision, natural language processing,
game playing,
protein structure prediction,
and many other fields.

A major weakness of DNNs is the existence of adversarial samples~\cite{S2013}, that is, it is possible to intentionally make small modifications to an input such that human can still recognize the input clearly, but the DNN outputs a wrong label or even any label given by the adversary.
Existence of adversary samples makes the DNN vulnerable in safety-critical applications.
Although many effective methods to defend adversaries were proposed~\cite{M2017,sur1,sur2,rubust-rev1}, it was shown that adversaries are still inevitable for current DNNs~\cite{asulay1,adv-inev1}.
In ~\cite{Bast1}, it was proved that adversarial attacks
always exist for any successful DNNs under certain conditions.
In this paper, we present a new approach by using the bias part of the DNN as the classifier
and show that the bias classifier is safe against gradient-based attacks.

\subsection{Contributions}

Let $\I=[0,  1]\subset\R$ and
 $\F:\I^n\to \R^m$   a classification DNN for $m$ objects, using Relu as the activation function.
For any $x\in \I^n$,  there exist  $W_x\in\R^{m\times n}$ and $B_x\in\R^{m}$    such that $$\F(x)=W_x x+B_x$$
where $W_xx$ is called the {\em first-degree part} and $B_x$ the {\em bias part} of $\F$.
From the definition of Relu,
the bias part
$$B_\F:\I^n\to \R^m$$
defined as $B_\F(x)=B_x$ is
a piecewise constant function with a finite number of values.

The most popular and effective  methods to generate adversaries,
such as  FGSM~\cite{G2014} or PGD~\cite{M2017}, use $\frac{\nabla\F(x)}{\nabla x}$ to make the loss function bigger.
An attack on DNNs only using the values of $\F(x)$ and $\frac{\nabla\F(x)}{\nabla x}$ is called a {\em gradient-based  attack}.
Since $n$ is generally quite large, using $\frac{\nabla\F(x)}{\nabla x}$ to find
adversaries in the high-dimensional space $\R^n$ seems inevitable.

Motivated by the above observation,
the {\em bias classifier} is introduced in this paper,
that is, the bias part
$B_\F:\I^n\to \R^m$ of $\F$ is used to classify the $m$ objects.
Since $B_\F$ is a piecewise constant function, it has zero gradients
and is safe against direct gradient-based attacks.
The contributions of this paper are summarized below.

First, the existence of the bias classifier is proved.
Precisely, it is shown that for any classification problem,
there exists a DNN $\F$ such that its bias part $B_\F$ gives the correct
label  with  arbitrarily high probability.

Second, an effective training method for the bias classifier is proposed.
It is observed that the adversarial training method introduced in~\cite{M2017}
  significantly increases the classification power of the bias part.
Furthermore, using the adversarial training  to the loss function
$L_{\CE}(B_{\F}(x),  y)+\gamma L_{\CE}(\F(x), y)$
increases the classification power of the bias part
and decreases the classification power of first-degree part of $\F$,
and hence is used to train  the bias classifier.

Third, an {\em information-theoretically safe} bias classifier against gradient-based attacks is given.
A network $\F$ is called  information-theoretically safe against an attack $\A$,
if when generate an adversary for a sample $x$ with $\A$,
a random attack direction is given.
In other words,  the rate to generate adversaries with $\A$ equals
the rate of random samples to be adversaries.
Let $W\in\R^{m\times n}$ be a matrix with certain random entries,
$\F:\I^n\to\R^m$ a trained bias classifier,
and $\widetilde{\F}(x) = \F(x) + Wx$.
Then, it is shown that   $B_{\widetilde{\F}}$ is information-theoretically safe against the gradient-based attack of $\widetilde{\F}$, if the structure and parameters of $\F$ are kept secret.
The notion of {\em information-theoretically safe}
is borrowed from cryptography~\cite{it},  which means
that the ciphertext yields no information regarding the plaintext
for cyphers which are perfectly random.

Fourth, several methods to attack the bias classifier are proposed.
Experiments with MNIST and CIFAR-10 show that
the bias classifier has comparable accuracies with DNNs on the test sets
and
is more robust than DNNs of similar sizes against these adversarial attacks in most cases.

\subsection{Related work}

There exist two main approaches to obtain more robust DNNs:
using a better training method
or a better structured  DNN.
Of course, the two approaches can be combined.

Many effective methods were proposed to train more robust DNNs to defend adversaries~\cite{sur1,sur2,rubust-rev1,sur-adv}.
The adversarial training method proposed by Madry et al~\cite{M2017}
can reduce adversaries significantly,
where the value of the loss function of the worst adversary in a small neighborhood of
the training sample is minimized.
A similar approach is to generate adversaries and
add them to the training set~\cite{G2014}.
A fast adversarial training algorithm was proposed,
which improves the training efficiency by reusing the backward pass calculations~\cite{adv-li3}.
A less direct approach to  resist adversaries is to make the DNN more stable
by introducing the Lipschitz constant or $L_{p,  \infty}$ regulations of each
layer~\cite{adv-li1,  S2013, YL1, Y2018}.
Adding noises to the training data is an effective way to increase
the robustness~\cite{DL}.
Knowledge distilling is also used to enhance  robustness and defend adversaries~\cite{H2015}.
%
%In \cite{rob-b3,  W2019,  rob-b1,  rob-b2},   methods to compute the robust regions of a DNN were given.

In this paper, the adversarial training~\cite{M2017} is
used to a new loss function to train the bias classifier.
%Also, many interesting properties of the adversarial training method~\cite{M2017}
%are observed with numerical experiments, which are given in section \ref{sec-b12}
%of this paper.

Many effective new structures for DNNs were proposed to defend adversaries.
%~\cite{sur1,sur2,rubust-rev1,sur-adv}.
%
The ensembler adversarial training~\cite{adv-li2} was introduced
for  CNN models,   which can apply to large datasets such as ImageNet.
In \cite{fd}, a denoising layer  is added to each hidden layer to defend adversarial attack.
In \cite{diff}, difference-privacy noise layers are added   to defend adversaries.
In \cite{lrank}, a low-rank DNN is shown to be more robust.
In \cite{netb}, a classification-autoencoder was proposed, which is robust against
outliers and adversaries.
In \cite{pj}, it was observed that by taking average values of points
in a small neighbourhood of an input can give a larger robust region for the input.
In \cite{z1,z2}, strategies to defend adversarial attacks by modifying the input   were given.

In this paper, a new idea to obtain robust DNNs is given, that is, the
bias part is used as the classifier to avoid gradient-based attacks.
Another advantage of using the bias part as the classifier is that,
an information-theoretically safe classifier can be constructed.
Our network does not deliberately hide the gradient like the method in~\cite{not}.
Our network does not have gradient,  so the white box attack method for the gradient hiding method in~\cite{not} does not work for our model.

\vskip10pt
The rest of this paper is organized as follows.
In section \ref{sec-hnet},  the existence of the bias classifier is proved
and the training method is given.
In section \ref{sec-safe}, several attack methods for the bias classifier
are given.
In section \ref{sec-isafe},  the bias classifier is shown to be
information-theoretically safe against the original-model gradient-based attack.
In section \ref{sec-exp},  numerical experimental results are given
to show  that the bias classifier  indeed improves robustness to resist adversaries.
In section \ref{sec-conc},   conclusions are given.

\section{Bias Classifier}
\label{sec-hnet}

In this section,   we prove the existence of a DNN $\F$ such that
the bias part of $\F$ can be used as a classifier.
We also give a training algorithm for the new classifier.

\subsection{The standard DNN}

Let $\I=[0,  1]\subset\R$ and $[n] = \{1, \ldots, n\}$ for $n\in\N_{>0}$.
Let $\F:\I^n\to \R^m$ be a
classification DNN with $L$ hidden layers
and the label set $\L=[m]$.
Each hidden-layer of $\F$  uses $\Relu$ as activity functions
and the output layer does not have activity functions. We write $\F:\I^n\to\R^m$ as
\begin{equation}
\label{eq-dnn}
%%\small
\begin{array}{ll}
x_0\in\I^n,  n_0=n,  n_{L+1}=m,\\
x_{l}=\Relu(W_{l}x_{l-1}+b_{l}) \in \R^{n_{l}}, l\in[L], \\
%\quad W_{l}\in \R^{n_l\times n_{l-1}},   b_{l}\in \R^{n_l},  l\in[L]; \\
%
 \F(x_0)=x_{L+1}=W_{L+1}x_{L}+b_{L+1},\\
%\quad W_{L+1}\in \R^{m\times n},   b_{L+1}\in \R^{m}\\
\end{array}
\end{equation}
where $W_{l}\in \R^{n_l\times n_{l-1}},   b_{l}\in \R^{n_l}$.
%$x_0\in\I^n,  n_0=n,  n_{L+1}=m$.
Denote $\Theta_\F=\{W_{l}, b_{l}\}_{l=1}^{L+1}$ to be the parameter set of $\F$.
Given a training set $\S$,  the network $\F$ can be trained by solving the following optimization problem with BP
\begin{equation}
\label{eq-NT}
\min_{\Theta}
\sum_{(x,  y)\in \S}L_{\CE}(\F(x),  y).
\end{equation}

For any $x\in \I^n$,   there exist  $W_x\in\R^{m\times n}$ and $B_x\in\R^{m}$,    such that $\F(x)=W_xx+B_x$.
We define the {\em first-degree part} of $\F$ to be $W_{\F}:\I^n\to \R^m$,   that is $W_{\F}(x)=W_x x$; and the {\em bias part} of $\F$ to be $B_{\F}:\I^n\to\R^m$,   that is $B_{\F}(x)=B_x$.
It is easy to see that
\begin{equation}
\label{eq-dnn1}
\F(x)=W_{\F}(x)+B_{\F}(x) = W_x x + B_x.
\end{equation}
%is valid for all $x\in\I^n$.
For a label $i\in[m]$ and $x\in\I^n$,    denote $\F_i(x)$ to be the $i$-th coordinate of
$\F(x)$.

A {\em linear region} of $\F$ is a maximal connected open subset of the input space $\I^n$,   on which $\F$ is linear~\cite{G2014}.
On each linear region  $A$ of $\F$,   there exist  $W_A\in\R^{m\times n}$ and $B_A\in\R^{m}$,   such that $\F(x)=W_Ax+B_A$ for $x\in A$.
Due to the property of Relu function,  it is clear that $\F$ has a finite number of disjoint linear regions
and $\I^n$ is the union of the closures of these linear regions.

\subsection{Existence of bias classifier}
\label{sec-hnet2}
In this section,  we will prove the existence of the bias classifier.
Let $\O\subset \I^n$ be the objects to be classified.
For $x\in\O$  and $r\in\R_{>0}$,  when $r$ is small enough,    all images in
$$\B(x,  r) =\{x+ \eta\,  |\,   \eta\in\R^n,   ||\eta||< r\}$$
can be considered to have the same label with $x$.
Therefore,   the object $\O$ to be classified may be considered
as bounded open sets in $\I^n$.
This observation motivates the following existence theorem,   whose proof is given in  Appendix A.
%For $D\subset \R^n$,   denote $V(D)$ to be the volume of $D$.
\begin{restatable}{theorem}{gtthmE}
%\begin{theorem}
\label{th-E1}
Let $\O=\bigcup_{i=1}^{m}O_i\subset\I^n$ be the elements to be classified
and $\L=\{l\}_{l=1}^m$ the label set,
where $O_i\subset \I^n$ is an open set,
$O_i\bigcap O_j=\phi$ if $i\neq j$,
and $x$ has label $l$ for $x\in O_l$.
Then for any $\epsilon>0$,   there exist  a DNN $\F$ and an open set $D\subset \I^n$ with volume $V(D)<\epsilon$,
such that $B_{\F}(x)$ gives the correct label for $x\in \O\setminus D$,
that is, the $l$-th coordinate of $\F(x)$ has the biggest value for $x\in O_l$.
%,   that is,
%the $l$-th coordinate of $B_{\F}(x)$ has the largest value among all coordinates
%of $B_{\F}(x)$ for $x\in O_l\setminus D$.
%\end{theorem}
\end{restatable}

A network $\F$ satisfying the conditions of Theorem \ref{th-E1} gives
a {\em bias classifier} $B_\F$, which can be computed from $\F$ as follows:
%
%since $W_{\F}(x)=\frac{\nabla\F(x)}{\nabla x}\cdot x$,
%we have
\begin{equation}
\label{eq-bias}
B_{\F}(x)=\F(x)-W_{\F}(x)  =\F(x)- \frac{\nabla\F(x)}{\nabla x}\cdot x.
\end{equation}

\subsection{Training  the bias classifier}
\label{sec-hnet4}

In order to increase the robustness of the network, we will use
the adversarial training introduced in~\cite{M2017}, which is one of the
best practical training method to defend adversaries.
Let $(x,  y)$ be a data in the training set $\S$.
Then the adversarial training is to solve
\begin{equation}
\label{eq-AT}
\min_{\Theta}\max_{||\zeta||<\varepsilon}
\sum_{(x,  y)\in \S}L_{\CE}(\F(x+\zeta),  y)
\end{equation}
where $\varepsilon\in\R_{>0}$ is a given small real number.
In order to increase the power of the bias part $B_\F$,  we  use the following training method
{%\scriptsize
\begin{equation}
\label{eq-HT1}
\min_{\Theta}\max_{||\zeta||<\varepsilon} \sum_{(x,  y)\in \S}
 [L_{\CE}(B_{\F}(x+\zeta),  y)+\gamma L_{\CE}(\F(x+\zeta),  y)]
\end{equation}
}
where $\gamma$ is a super parameter.
The training procedure is given in Algorithm \ref{alg-ht1}.

We first use a simple example to show that the adversarial training can increase
the classification power of $B_\F$.
The accuracies of $\F$, $W_{\F}$,  $B_{\F}$ on  the test set for
three kinds of training methods are given in Table \ref{tab-rem21}, respectively.
More comprehensive numerical experiments are given in section \ref{sec-exp}.

\begin{table}[H]
\centering
\begin{tabular}{|c|c|c|c|}
  \hline

   & $W_{\F}$  &$B_{\F}$ & $\F$ \\ \hline
  Normal training \eqref{eq-NT} & 98.80$\%$& $15.62\%$& 99.09$\%$\\
  %The accuracy on Test set of CIFAR-10(without AT)& 28.69$\%$&85.29$\%$& $90.03\%$ \\
  Adversarial training \eqref{eq-AT} & 90.61$\%$& $98.77\%$ &99.19$\%$\\
 Adversarial training \eqref{eq-HT1}& 0.28$\%$& 99.09 $\%$ &99.43$\%$\\
  %The accuracy on Test set of CIFAR-10 (with AT)& 30.20$\%$& $61.72\%$ &80.61$\%$\\
  \hline
\end{tabular}
\caption{Accuracies of network Lenet-5 for  MNIST}
\label{tab-rem21}
\end{table}

\begin{algorithm}[H]
\caption{BCTrain}
\label{alg-ht1}
\begin{algorithmic}
\REQUIRE ~~\\
The set of training data:   $\S=\{(x_i,  y_i)\}$;\\
The initial value of the parameter set $\Theta$:  $\Theta_{0}$;\\
The super parameter: $M_s,   M_b,   M_n$.\\
\ENSURE The trained parameters $\widetilde{\Theta}$.~~\\
%The trained parameter $\Theta_{l}.
%
In each iteration:\\
Input $\Theta_{k}$\\
Let $L(x,  y,  \Theta)=L_{\CE}(B_{\F_{\Theta}}(x),  y)$.\\
Let $L_1(x,  y,  \Theta)=L_{\CE}(\F_{\Theta}(x),  y)$.\\
For $(x,  y)\in \S$,   do\\
%\begin{eqnarray }
\quad i=0,  $x_0=x$\\
\quad While $i<M_s$:\\
\quad\quad$x_{i+1}=x_i-M_b \frac{\partial L(x_i,  y,  \Theta_{k})}{\partial x_i}$\\
\quad\quad$i=i+1$\\
\quad $x=x_{i+1}$\\
Let $L(\Theta_k)=\frac{1}{|\S|}\sum_{(x,y)\in \S}L(x,  y,  \Theta_k)$.\\
Let $L_1(\Theta_k)=\frac{1}{|\S|}\sum_{(x,y)\in \S}L_1(x,  y,  \Theta_k)$.\\
Let $\bigtriangledown L=\frac{\partial (L(\Theta_k)+M_n L_1(\Theta_k))}{\partial \Theta_k}$.\\
%\end{eqnarray }
Output $\Theta_{k+1}=\Theta_k+\gamma_{k}\bigtriangledown L$; $\gamma_{k}$ is the stepsize at iteration k. \\
\end{algorithmic}
\end{algorithm}

\section{Attack methods for the bias classifier}
\label{sec-safe}
In this section, several possible methods to attack the bias classifier are  given.

\subsection{Safety against gradient-based attack}
\label{sec-safe1}

The most popular methods to generate adversaries,  such as  FGSM~\cite{G2014} or PGD~\cite{M2017},  use $\frac{\nabla\F(x)}{\nabla x}$ to make the loss function bigger.
More precisely,  adversaries are generated as follows
\begin{equation}
\label{eq-gba}
x\to x+\varepsilon \sign(\frac{\nabla L_{CE}(\F(x),y)}{\nabla x})
%x \to x + \varepsilon \sign( \frac{\nabla \Loss_\F(x) }{\nabla x})
\end{equation}
for a small parameter  $\varepsilon\in\R_{>0}$.
%Where $\Loss_(\F)(x)$ is the loss function. Easy to see that,  when we get the $\F(x)$ and
%
It is easy to see that,  $\frac{\nabla L_{CE}(\F(x),y)}{\nabla x}$
can be obtained from $\frac{\nabla\F(x)}{\nabla x}$.
%for most loss functions such as $L_{CE}$.
%
So,  in the above attack,  only the values of  $\F(x)$ and
$\frac{\nabla  \F(x) }{\nabla x}$ are needed
and the detailed structure of $\F$ is not needed.
%
%This kind DNN model is called a {\em gradient-based  model}
%and the corresponding attack is called {\em gradient-based  attack}.
%
%Inspired by this,  between the black-box or white-box DNN models,  we consider a model in between.
Motivated by this fact, we introduce the concept of gradient-based  attack.
A DNN model is called a {\em gradient-based  model},
if for $x\in\I^n$,  the values of $\F(x)$ and $\frac{\nabla\F(x)}{\nabla x}$
are known,  but the detailed structure of $\F$ is not known.
Correspondingly,  an attack only uses the values of $\F(x)$ and $\frac{\nabla\F(x)}{\nabla x}$ is called {\em gradient-based  attack}.

Since  the derivative of $B_\F$ is always zero,
a gradient-based  attack against $B_\F$ becomes a black-box attack,
and in this sense we say that the bias classifier is safe against the gradient-based  attack.

In the gradient-based  model,  we do not know the structure of $\F$,   but we  can calculate $B_{\F}(x)$ from $\frac{\nabla\F(x)}{\nabla x}$ using \eqref{eq-bias},
and the bias classifier still works.

\subsection{Original-model attack}
\label{sec-exp20}
An obvious attack for the bias classifier is to create
adversaries of $B_\F$ using the gradients of  $\F$, which is called {\em original-model attack}. The attack is given in Algorithm 2, where $\widehat{B}_{\F}(x)$ is  the label of $B_{\F}(x)$.

\begin{algorithm}[H]
\caption{OAttack}
\label{alg-weak}
\begin{algorithmic}
\REQUIRE ~~\\
The value of the parameter set $\Theta$ of $\F$;\\
The super parameters: $\epsilon\in\R $,  $N\in \N$;\\
A sample $x_0$ and its label $y_0$.\\
\ENSURE An adversarial sample $x_a$.\\
$x=x_0$\\
For $i=1,\ldots,N$:\\
\quad If $\widehat{B}_{\F}(x)\ne y_0$:\\
\quad\quad Break.\\
\quad $x=x+\epsilon \sign(\frac{\nabla L_{CE}(\F(x),y)}{\nabla x})$\\
If $\widehat{B}_{\F}(x)\ne y_0$, $x_a=x$  output: $x_a$\\
Output: No adversary for $x_0$\\
\end{algorithmic}
\end{algorithm}

\subsection{Correlation attack on the bias classifier}
\label{sec-safe4}
From   numerical experiments,  we have the following observations.
For a network $\F:\I^n\to\R^m$ trained with \eqref{eq-HT1} and a small vector $\epsilon\in\R^n$,
the following fact happens with high probability:
$W_{\F}(x)[l]\ge W_{\F}(x')[l]$ is valid
if and only if $B_{\F}(x)[l]\le B_{\F}(x')[l]$ is valid,
where $l\in\L$ and $x'=x+\epsilon$.
In other words,   $W_{\F}$ and $B_{\F}$ are co-related
and we thus can decrease $B_\F[l]$ by increasing $W_{\F}[l]$,
which is called the {\em correlation attack}.

In the correlation attack, we create adversaries by making
$W_{\F}(x)[y]-W_\F(x)[i]$ bigger, where $y$ is the label of $x$,  $i\in[m]$ and $i\ne y$.
The attack is given in Algorithm 3.

\begin{algorithm}[H]
\caption{CAttack}
\label{alg:Framwork}
\begin{algorithmic}
\REQUIRE ~~\\
The value of the parameter set $\Theta$ of $\F$;\\
The super parameters: $\epsilon\in\R $,   $N\in \N$;\\
A samples $x_0$ and its label $y$.\\
\ENSURE An adversarial sample $x_a$.\\
For $i\in\L$ and $i\ne y$:\\
\quad $x=x_0$,  j=0\\
\quad While $j<N$:\\
\quad\quad $U_{a}=\frac{\nabla\F_y(x)}{\nabla x}-\frac{\nabla\F_i(x)}{\nabla x}$, $x=x+\epsilon U_{a}$\\
\quad\quad If $\widehat{B}_{\F}(x)\ne y$,  break; else: j=j+1\\
%\quad if $j=N$,  break\\
If $\widehat{B}_{\F}(x)\ne y$,  $x_a=x$  output: $x_a$\\
Output: No adversary for $x_0$\\
\end{algorithmic}
\end{algorithm}

\section{Information-theoretically safety against original-model gradient-based  attack}
\label{sec-isafe}
By the original-model gradient-based attack,  we mean using the gradient of $\F$
to generate adversaries for $B_\F$.
In this section,  we show that it is possible to make
the bias classifier safe against this kind of attack.
%Which means, use gradient-based attack is almost equal to random select point around the samples.
The idea is to make $\frac{\nabla\F(x)}{\nabla x}$ random and
$B_\F$ still gives the correct classification.
%
%In this section,  two kinds of gradient-based  attacks are considered:
%the direct attack~\cite{dpf} and FGSM.

\subsection{Information-theoretically safety}
\label{sec-isafe0}
In this section, we will define the concept of
information-theoretically safety of a DNN against an attack.
%
%In the rest of this section, we give an equivalent criterion
%for the concept of information-theoretically safety in terms of adversary creation rates.

Let  $\F$ be a DNN defined in \eqref{eq-dnn}.
Motivated by the FGSM attack \eqref{eq-gba}, we assume that
the attack $\A(x,\F,\rho): \R^n\to\R^n$ generates an adversary of $x$ as below:
\begin{equation}
\label{eq-at00}
\A(x,{\F},\rho)=x+ \rho\, V
\end{equation}
where $\rho\in\R_{>0}$ and $V\in\{-1,1\}^n$ is the sign vector of certain quantity related with the gradient of $\F(x)$.

The attack $\A(x,{\F},\rho)$ is called
{\em information-theoretically safe}, if
$V= (\A(x,{\F},\rho)-x)/\rho$ is a random vector in $\{-1,1\}^n$
for any input $x$.

We now show how to build an information-theoretically safe bias classifier.
First train a DNN $\F:\I^n\to\R^m$ with the method in Section \ref{sec-hnet4}.
Let $W_R\in\R^{m\times n}$ satisfy
a given distribution $\M$ of random matrices in $\R^{m,n}$ and
\begin{equation}
\label{eq-gnet1}
\begin{array}{l}
\widetilde{\F}(x) =  \F(x) + W_Rx =  (W_x + W_R)x+B_x \\
B_{\widetilde{\F}}(x)=\widetilde{\F}(x)- \frac{\nabla{\widetilde{\F}(x)}}{\nabla x}\cdot x.
\end{array}
\end{equation}
%which is called {\em randomized version} of $\F$.
%
It is easy to see that $B_{\widetilde{\F}} = B_{\F}$,
that is, the bias classifiers for $\F$ and $\widetilde{\F}$ are the same.
On the other hand, $\frac{\nabla {\widetilde{\F}}(x)}{\nabla x} =
\frac{\nabla \F(x)}{\nabla x}+W_R$ is random in certain sense.

The safety of $B_{\widetilde{\F}}$ against the  attack
$\A(x,{\widetilde{\F}},\rho)$
 can be measured by the following adversary creation rate
\begin{equation}
\label{eq-sacr0}
%\tiny
%%\scriptsize
\begin{array}{l}
\Ca(B_{\widetilde{\F}},\A,\M)
=\E_{W_R\sim \M}[\E_{x\sim \D_{\O}}[\ID(\widehat{B}_{{\F}}(\A(x,\widetilde{\F},\rho))\ne\widehat{B}_{{\F}}(x))]]
\end{array}
\end{equation}
where $\widehat{B}_{\F}$ is the label of the classification
%$\varepsilon$ is the given distance that before,
and $\D_{\O}$ is the distribution of the objects to be classified.

If  $B_{\widetilde{\F}}$ is
information-theoretically safe against the attack $\A(x,{\widetilde{\F}},\rho)$,
then it is easy to show that $\Ca(B_{\widetilde{\F}},\A,\M)$ equals
\begin{equation}
\label{eq-sacr1}
%\tiny
%%\scriptsize
\begin{array}{l}
\Ca(\F,\rho)=\frac{1}{2^n} \E_{x\sim \D_{\O}}\sum_{V \in\{-1,1\}^n}[\ID(\widehat{B}_{\F}(x+\rho\,V)\ne \widehat{B}_{\F}(x))]
\end{array}
\end{equation}
which depends only on $\F$ and $\rho$ and will be used as a measure of the robustness of the bias classifier.

Note that $\Ca(\F,\rho)$ is the rate of adversaries in certain random samples.
%As shown in section \ref{sec-exp31},
%$\Ca(\F,\rho)$ is small for most networks when $\rho$ is small.
%
In other words, if $B_{\widetilde{\F}}$ is information-theoretically safe under attack $\A$,
then the adversary creation rate of $B_{\widetilde{\F}}$ under attack $\A$
is equal to the rate of random samples to be adversaries,
which is very small as shown in section \ref{sec-exp31}.
%$\Ca(\F,\rho)$ is small for most networks when $\rho$ is small.

If $B_{\widetilde{\F}}$ is not information-theoretically safe, we can use the value
$\Ca(B_{\widetilde{\F}},\A,\M)/\Ca(\F,\rho)$ to measure
the safety of $B_{\widetilde{\F}}$ relative to the information-theoretically safety.

\subsection{Safety against direct attack}
\label{sec-isafe1}
In this section, we show that  $B_{\widetilde{\F}}$ defined
in \eqref{eq-gnet1} is safe
against the direct attack~\cite{dpf} of $\widetilde{\F}$.

Let $\U(a,b)$ be the uniform distribution in $[a,b]\subset\R$.
For $\lambda\in\R_{>0}$, denote $\M_{m,n}(\lambda)$ to be the  random matrices
such that the elements of their $i$-row are in
$(\U(-2i\lambda ,-(2i-1)\lambda )\cup\U((2i-1)\lambda ,2i\lambda))^{m\times n}.$

Let $||x||_{-\infty}=\min_{i\in[n]}\{|x_i|\}$ for $x\in\R^n$.
It is easy to see that for $W_R\sim \M_{m,n}(\lambda)$,
we have $||W_{R,i}-W_{R,j}||_{-\infty}>\lambda $ for $i\ne j$,
where $W_{R,i}$ is the $i$-th row of $W_R$.

For $\rho\in\R_{>0}$,
consider the following gradient-based {\em direct attack} \cite{dpf}   for
the network ${\F}$:
\begin{equation}
\label{eq-at11}
\begin{array}{ll}
\A_1(x,{\F},\rho)=x+ \rho\, \sign(\frac{\nabla {\F}_{n_x}(x)}{\nabla x}-\frac{\nabla {\F}_y(x)}{\nabla x})\\
\end{array}
\end{equation}
where $y$ is the label of $x$
and $n_x=\arg\max_{i\ne y}\{{\F}_i(x)\}$.

\begin{theorem}
\label{th-safe1}
Let  $|\frac{\nabla \F(x)}{\nabla x}|_{\infty}<\lambda/2 $
and $W_R\in\M_{m,n}(\lambda)$.
If the structure and parameters of $\F$ are kept secret, then
$B_{\widetilde{\F}}$ is information-theoretically safe against the  attack
$\A_1(x,\widetilde{\F},\rho)$.
% in the sense that
%attacking $B_{\widetilde{\F}}$ using $\A_1(x,\widetilde{\F},\rho)$
%will generate a random direction $\A_1(x,{\F},\rho)-x$ for generating adversaries.
\end{theorem}
%The proof of Theorem \ref{th-safe1} is given in the appendix.
%
\begin{proof}
From \eqref{eq-dnn1} and \eqref{eq-gnet1},
$\frac{\nabla\widetilde{\F}(x)}{\nabla x}
=W_x+W_R$.
Let $W_{R,i}$ and $W_{x,i}$ be the $i$-rows of $W_{R}$ and $W_{x}$, respectively.
If $W_R\sim \M_{m,n}(\lambda)$, then  $||W_{R,i}-W_{R,j}||_{-\infty}>\lambda $ for $i\ne j$.
Since   $|\frac{\nabla \F(x)}{\nabla x}|_{\infty}=|W_x|_{\infty}<\lambda/2 $, we have
$||W_{x,i}-W_{x,j}||_{\infty}<\lambda $ for $i\ne j$.
Then, %from \eqref{eq-at11},
{%\small
\begin{equation}
\label{eq-sf11}
\begin{array}{ll}
&\A_1(x,\widetilde{\F},\rho)\\
&=x+\rho\, \sign(\frac{\nabla \widetilde{\F}_{n_x}(x)}{\nabla x}-\frac{\nabla \widetilde{\F}_y(x)}{\nabla x})\\
&=x+\rho\, \sign(W_{x,n_x}-W_{x,y} + W_{R,n_x}-W_{R,y})\\
&=x+\rho\, \sign(W_{R,n_x}-W_{R,y}).\\
\end{array}
\end{equation}
}
Since $W_R\in\M_{m,n}(\lambda)$, $\widehat{W}=W_{R,n_x}-W_{R,y}$ is a random vector whose entries
having values in two intervals of the form $[-b_2,-b_1]\cup [b_1,b_2]$,
$\sign(\widehat{W})$ is a random vector in $\{-1,1\}^n$
and the theorem is proved.
\end{proof}

\subsection{Safety against FGSM attack}
\label{sec-isafe2}
In this section, we show that the result in section \ref{sec-isafe1} holds for the  FGSM attack if $m=2$.
Here is the  FGSM attack:
\begin{equation}
\label{eq-aa2}
\A_2(x,\F,\rho)=x+\rho\,\sign(\frac{\nabla L({\F}(x),y)}{\nabla x}).
\end{equation}

\begin{theorem}
\label{th-safe2}
If $|\frac{\nabla \F(x)}{\nabla x}|_{\infty}<\lambda/2$, $W_R\sim \M_{m,n}(\lambda)$,
and $m=2$, then
$B_{\widetilde{\F}}$ is information-theoretically safe against the  attack
$\A_2(x,\widetilde{\F},\rho)$.
\end{theorem}
\begin{proof}
Let $y\in\{0,1\}$ be the label of $x$.
Use the notations introduced in the proof of Theorem \ref{th-safe1}.
Since the loss function is  $L_{CE}$ and $m=2$, we have
% Error
%
\begin{equation}
\label{eq-fg11}
%\scriptsize
\begin{array}{ll}
&\frac{\nabla L(\widetilde{\F}(x),y)}{\nabla x}\\
=&\frac{\sum_{i=1}^m  {e^{\widetilde{\F}_{i}}}(W_{x,i}-W_{x,y} + W_{R,i}-W_{R,y})}{\sum_{i=1}^m  e^{\widetilde{\F}_{i}(x)}}\\
%&=\frac{1}{\sum_{i=1}^m  e^{\widetilde{\F}_{i}(x)}}
%(\sum_{i=1}^m  e^{\widetilde{\F}_{i}}(W_{x,i}-W_{x,y})+
% \sum_{i=1}^m  e^{\widetilde{\F}_{i}}( W_{R,i}-W_{R,y}))\\
=&\frac{e^{\widetilde{\F}_{1-y}(x)}}{\sum_{i=1}^m  e^{\widetilde{\F}_{i}(x)}}
(W_{x,1-y}-W_{x,y} + W_{R,1-y}-W_{R,y}).\\
\end{array}
\end{equation}
The last equality comes from $m=2$.
Since $||W_{x,i}-W_{x,j}||_{\infty}<\lambda $ and $||W_{R,i}-W_{R,j}||_{-\infty}>\lambda$
for $i\ne j$,
we have
$\sign(\frac{\nabla L(\widetilde{\F}(x),y)}{\nabla x})
=\sign(W_{R,1-y}-W_{R,y})$ which is a random vector in $\{-1,1\}^n$,
similar to the proof of Theorem \ref{th-safe1}.
The theorem is proved.
\end{proof}

When $m>2$, we have the following result, whose proof is given in Appendix B.
\begin{restatable}{theorem}{gtthmo}
%
%\begin{theorem}
\label{th-safe6}
Assume $|\frac{\nabla \F(x)}{\nabla x}|_{\infty}<\mu/2$,
 $|B_{\F}(x)|_{\infty}<\beta$,
% $P_{x\sim D_{\O}}(|x|_{\infty}=1)=1$,
and $\lambda\in\R_{>0}$ satisfying $(\lambda-\mu)e^{-2\beta-n\mu+\sqrt{\lambda}}>(2m\lambda+\mu)m$.
Furthermore, assume the samples are normalized, that is, $|x|_{\infty}=1$.
If $W_R\sim \M_{m,n}(\lambda)$, then
$\Ca(B_{\widetilde{\F}},\A_2,\M_{m,n}(\lambda))\le (m-1)\Ca(\F,\rho)+\frac{(m-2)^2}{\sqrt{\lambda}}$.
%\end{theorem}
\end{restatable}
We can choose a  large $\lambda$ to make the term $\frac{(m-2)^2}{2\sqrt{\lambda} \eta}$ small. So from Theorem \ref{th-safe6}, $B_{\widetilde{\F}}$ is approximately safe if $m$ is small.
%
%For the MNIST dataset, $m=10$ and $\eta_1$ is near zero.

\subsection{Safety against direct attack under simpler distribution}
\label{sec-isafe3}
Let $\U_{m,n}(\lambda)$ be the random matrices whose entries are in $\U(-\lambda ,\lambda )$.
In this section, we show that the result in section \ref{sec-isafe1}
is approximately valid for the simpler distribution $\U_{m,n}(\lambda)$.
We consider the $k$-step  direct attack:
{%%\small
%%\scriptsize
%\tiny
\begin{equation}
\label{eq-A3}
\begin{array}{l}
x^{(0)}=x\\
x^{(i)}=x^{(i-1)}+\frac{\rho}{k}\,\sign(\frac{\nabla \F_{n_x}(x^{(i-1)})}{\nabla x^{(i-1)}}-\frac{\nabla \F_y(x^{(i-1)})}{\nabla x^{(i-1)}}), i\in[k]\\
\A_3(x,\F,\rho)=x^{(k)}\\
\end{array}
\end{equation}
}
where $y$ is the label of $x$  and $n_x=\arg\max_{i\ne y}\{\F_i(x)\}$.

\begin{restatable}{theorem}{gtthmS}
%\begin{theorem}
\label{th-safe3}
If $|\frac{\nabla \F(x)}{\nabla x}|_{\infty}<\mu/2$
and $W_R\sim \U_{m,n}(\lambda)$, then
$\Ca(B_{\widetilde{\F}},\A_3,\U_{m,n}(\lambda))\le\Ca(\F,\rho)+\mu n/\lambda$.
Furthermore, if $\lambda >n\mu/\epsilon$, then $\Ca(B_{\widetilde{\F}},\A_3,\U_{m,n}(\lambda))\le\Ca(\F,\rho)+\epsilon$
for any $\epsilon\in\R_{>0}$,
and in particular, if $\lambda >na/(\epsilon\Ca(\F,\rho))$, then $\Ca(B_{\widetilde{\F}},\A_3,\U_{m,n}(\lambda))\le(1+\epsilon)\Ca(\F,\rho)$.
%\end{theorem}
\end{restatable}
Proof of Theorem \ref{th-safe3} is given in  Appendix C.
Theorem \ref{th-safe3} implies that $B_{\widetilde{\F}}$ can be made close to information-theoretically safe under attack $\A_3(x,\widetilde{\F},\rho)$.

\subsection{Safety against FGSM under simpler distribution}
\label{sec-isafe4}
In this section, we show that the result in section \ref{sec-isafe2}
is approximately valid for the simpler distribution $\U_{m,n}(\lambda)$.
Let $\A_2$ be the attack in \eqref{eq-aa2}. Then we have
\begin{restatable}{theorem}{gtthmP}
%\begin{theorem}
\label{th-safe4}
If $|\frac{\nabla \F(x)}{\nabla x}|_{\infty}<\mu/2$,
$W_R\sim \U_{m,n}(\lambda)$, and $m=2$, then
$\Ca(B_{\widetilde{\F}},\A_2,\U_{m,n}(\lambda))\le e^{n\mu/\lambda}\Ca(\F,\rho)$.
Furthermore, if $\lambda>n\mu/\ln(1+\epsilon)$,
then $\Ca(B_{\widetilde{\F}},\A_2,\U_{m,n}(\lambda))\le(1+\epsilon)\Ca(\F,\rho)$.
%\end{theorem}
\end{restatable}

For the general $m$, we have
\begin{restatable}{theorem}{gtthmQ}
%\begin{theorem}
\label{th-safe5}
Assume $|\frac{\nabla \F(x)}{\nabla x}|_{\infty}<\mu/4$, $|B_{\F}(x)|_{\infty}<\beta$, %$P_{x\sim D_{\O}}(|x|_{\infty}=1)=1$
and  $\lambda\in\R_{>0}$ satisfying $\mu e^{-2\beta-n\mu/2+\sqrt{\lambda}}>2(2\lambda+\mu)m$.
Furthermore, assume the samples are normalized, that is, $|x|_{\infty}=1$.
If $W_R\sim \U_{m,n}(\lambda)$, then
$\Ca(B_{\widetilde{\F}},\A_2,\U_{m,n}(\lambda))\le (m-1)\Ca(\F,\rho)+\frac{(m-1)n\mu}{\lambda}+\frac{(m-2)^2}{\sqrt{\lambda} }$.
%\end{theorem}
\end{restatable}

Proofs of Theorems \ref{th-safe4} and \ref{th-safe5} are given in
Appendixes D and E, respectively.
Theorem \ref{th-safe4} shows that, for binary classifications,  $B_{\widetilde{\F}}$ is close to information-theoretically safe against FGSM under distribution $\U_{m,n}(\lambda)$.
Theorem \ref{th-safe5} shows that the result is approximately valid  in the general case
under certain conditions.

\section{Experiments}
\label{sec-exp}

\subsection{Accuracy of the bias classifier}
\label{sec-exp1}
In this section,  we give the accuracy of the bias classifier using the MNIST and CIFAR-10 data sets.
%\footnote{For CIFAR-10, a normalization to the samples is carried out before training, which makes we do not use the conventional attack distance $\frac{8}{255}$ in the follow experiment.}.
%
We compare two DNN models:
\begin{equation}
\label{eq-mod11}
\begin{array}{l}
%=2pt\parindent=10pt
\F^{(1)}:\hbox{  trained with adversarial trianing \eqref{eq-AT}}\\
\F^{(2)}:\hbox{ trained with  Algorithm \ref{alg-ht1}}
\end{array}
\end{equation}
whose detailed structure can be found in Appendix F.

We give the  accuracy on the test set (TS) and  the strong adversaries  (SA, see~\cite{netb})
and the results are given in Table \ref{tab-1}.

From the table, we can see that the bias classifier has comparable
accuracies with  $\F^{(1)}$ on the test set,
but achieves significant higher accuracies than $\F^{(1)}$ for the strong adversaries,
which implies that the bias classifier is more robust against adversaries
than DNNs of similar size and trained with adversarial training.

\begin{table}[H]
%\tiny
\centering
\begin{tabular}{|c|c|c|c|c|}
  \hline
  % after \\: \hline or \cline{col1-col2} \cline{col3-col4} ...
  DNN & TS/MNIST  & SA/MNIST & TS/CIFAT-10  & SA/CIFAT-10
  \\\hline
  $\F^{(1)}$& 99.19$\%$ & 51.5$\%$
   & 81.23$\%$ & 19$\%$  \\
  $B_{\F^{(2)}}$& 99.12$\%$ & 87.5$\%$
  &82.84$\%$ &42$\%$\\
  %$B_{\F^{(2)}}$ & 99.09$\%$  &55.5$\%$
  %& 84.22$\%$ & 27$\%$\\
%  $B_{\F^{(2)}}$ with loss function \ref{eq-L3}& 99.02$\%$ & 63.9$\%$\\
\hline
\end{tabular}
\caption{Accuracies for MNIST and CIFAR-10 }
\label{tab-1}
\end{table}

Moreover, for CIFAR-10, we compare the accuracy of our network and
two other networks ResNet18 and VGG19, all using adversarial training.
From the results in Table \ref{tab-41},
our network $\F^{(1)}$ performs better than ResNet18 and VGG 19.

\begin{table}[H]
\centering
\begin{tabular}{|c|c|c|}
  \hline
  % after \\: \hline or \cline{col1-col2} \cline{col3-col4} ...
  DNN & Test Set& Strong Adversaries  \\\hline
  $\F^{(1)}$ & 81.23$\%$ & 19$\%$  \\
  ResNet18 &80.64$\%$ &9$\%$\\
  VGG19 & 78.92$\%$ &12$\%$\\
  \hline
\end{tabular}
\caption{Accuracies for three networks on CIFAR-10.}
\label{tab-41}
\end{table}

As pointed out in \cite{trades}, networks trained
with adversarial training usually have lower accuracies,
and the accuracies given in Tables \ref{tab-1} and \ref{tab-41} are about the
best ones for DNNs of similar sizes, which implies that the models
$\F^{(1)}$ and $\F^{(2)}$ are appropriate for MNIST and CIFAR-10.

\subsection{Robustness of the bias classifier against original-model attack}
\label{sec-oma1}
In this section,  we check the robustness of the bias classifier against the original-model attack given in Algorithm \ref{alg-weak}.
%,  that is,   the ability of $B_{\F}$ to defend adversaries of $\F$ trained
%with loss function \eqref{eq-HT1}.
%
\subsubsection{Experimental results}
\label{sec-oma11}

We use two more networks:
$\F^{(3)}$ has the same structure with $\F^{(1)}$  given in \eqref{eq-mod11},
but trained with the first-order regulation method~\cite{Fo},
and $\F^{(4)}$ has the same structure with $\F^{(1)}$,
but trained with TRADES~\cite{trades}.
Six kinds of adversaries are used:

{\bf $l_{\infty}$ adversaries}: $1$-$i$ $(i=1,  2,  3)$. Each pixel of the sample changes at most $0.i$. PGD~\cite{M2017} is used to attack:  each step changes 0.01 and moves $10i$ steps.

{\bf $l_{0}$ adversaries}: $2$-$i$ $(i=40,  60,  80)$. Change at most $i$ pixels of the sample. JSMA~\cite{JSMA} is used to attack:  change $i$ pixels and each pixel can change up to $1$.

The adversary creation rates  are given in Tables \ref{tab-mr1} and \ref{tab-cr1}.
%
%The first four rows are the adversary creation rates  for $\F^{(1)},\F^{(2)}, \F^{(3)}, \F^{(4)}$
%with method  PGD~\cite{M2017} for $l_{\infty}$ adversaries  and JSMA~\cite{JSMA} for $l_{0}$ adversaries.
%
%The $4$-th to the $6$-th rows are the percentages of adversaries of $\F^{(i)}$ that can fool $B_{\F^{(i)}}$.
%
The results in the last two rows are obtained with the original-model attack.
\begin{table}[H]
\centering
%\scriptsize
\begin{tabular}{|l|c|c|c|c|c|c|}
  \hline
  % after \\: \hline or \cline{col1-col2} \cline{col3-col4} ...
  DNN & 1-1 & 1-2 & 1-3 & 2-40 & 2-60 & 2-80 \\
 \hline
  $\F^{(1)}$ & $3\%$ & 17$\%$ & 55$\%$ & $55\%$ & $79\%$ & $87\%$ \\
  $\F^{(2)}$ & $4\%$ & $22\%$ & $77\%$ & $62\%$ & 82$\%$ & 90$\%$ \\
 % $\F^{(2)}$ & 3$\%$ & 14$\%$ & 66$\%$ & 73$\%$ & 87$\%$ & 95$\%$ \\
%
  $\F^{(3)}$ & $22\%$ & $78\%$ & $99\%$ & $75\%$ & 98$\%$ & 99$\%$ \\
  $\F^{(4)}$ & 4$\%$ & 15$\%$ & 53$\%$ & 62$\%$ & 77$\%$ & 88$\%$ \\
%  $B_{\F^{(1)}}$ & 62$\%$ & 40$\%$ & 21$\%$ & $57\%$ & $53\%$ & $50\%$ \\
%  $B_{\F^{(2)}}$ & $43\%$ & $11\%$ & $4\%$ & 63$\%$ & 60$\%$ & 59$\%$ \\
%  $B_{\F^{(2)}}$ & 52$\%$ & 49$\%$ & 38$\%$ & 55$\%$ & 51$\%$ & 50$\%$ \\
  $B_{\F^{(1)}}$ & 3$\%$ & 14$\%$ & 49$\%$ & $48\%$ & $67\%$ & $92\%$ \\
  $B_{\F^{(2)}}$ & $2\%$ & $6\%$ & $22\%$ & 41$\%$ & 56$\%$ & $79\%$ \\
%  $B_{\F^{(2)}}$ & 2$\%$ & 7$\%$ & 25$\%$ & 41$\%$ & 45$\%$ & 47$\%$ \\
% $B_{\F^{(2)}}$ with loss function $\ref{eq-L3}$ &2$\%$&3$\%$&7$\%$&36$\%$&45$\%$&50$\%$\\
  \hline
\end{tabular}
\caption{Creation rates of adversaries for MNIST}
\label{tab-mr1}
\end{table}

\begin{table}[H]
\centering
%\scriptsize
\begin{tabular}{|c|c|c|c|c|c|c|}
  \hline
  % after \\: \hline or \cline{col1-col2} \cline{col3-col4} ...
  DNN &        1-1              & 1-2            & 1-3      & 2-40        & 2-60         & 2-80 \\
\hline
  $\F^{(1)}$ &        $54\%$             & $77\%$           & $90\%$       & $72\%$        & $85\%$         & 96$\%$ \\
  $\F^{(2)}$ &      $54\%$            & $72\%$           & $85\%$        & $69\%$       & $88\%$        & 97$\%$ \\
 % $\F^{(2)}$ &   $47\%$             & $71\%$           & $84\%$       & $70\%$        & 88$\%$         & 94$\%$ \\
%
  $\F^{(3)}$ &      $88\%$            & $92\%$           & $99\%$        & $89\%$       & $99\%$        & 99$\%$ \\
  $\F^{(4)}$ &   $49\%$             & $73\%$           & $85\%$       & $70\%$        & 89$\%$         & 97$\%$ \\
  %$B_{\F^{(1)}}$ &    $73\%$             & $67\%$           & $63\%$       & $56\%$        & $64\%$         & $61\%$ \\
%  $B_{\F^{(2)}}$ &  $76\%$             & $66\%$           & $61\%$       & $67\%$         & 60$\%$      & 60$\%$ \\
%  $B_{\F^{(2)}}$& $72\%$            & $63\%$           & $59\%$       & $56\%$        & $57\%$         & $59\%$ \\
  $B_{\F^{(1)}}$ &    $67\%$             & $70\%$           & $86\%$       & $70\%$        & $84\%$         & $91\%$ \\
   $B_{\F^{(2)}}$     &$41\%$             & $58\%$           & $77\%$       & $49\%$      & $73\%$          &84$\%$ \\
 % $B_{\F^{(2)}}$ & $35\%$           & $47\%$           & $53\%$       & $41\%$        & $53\%$         & 59$\%$ \\
  \hline
\end{tabular}
\caption{Creation rates of adversaries for CIFAR-10}
\label{tab-cr1}
\end{table}

From the tables, we can see that the bias classifiers $B_{\F^{(2)}}$
has significant lower adversary creation rates  than all other networks.
For $l_{\infty}$ adversaries of MNIST,  $B_{\F^{(2)}}$ achieves near optimal results and the adversaries almost disappear.
For CIFAR-10,  the adversary creation rates  are still quite high  comparing to that of MNIST. We will explain the reason in Section \ref{sec-bx3}.

Also, $B_{\F^{(2)}}$ achieves much better results than $B_{\F^{(1)}}$,
which implies that our training method \eqref{eq-HT1}
is better than the usual adversarial training \eqref{eq-AT}.

\subsubsection{Influence of adversarial training on the bias classifier}
\label{sec-bx3}
In this section, we give an intuitive explanation for the results in
Tables \ref{tab-mr1} and \ref{tab-cr1}.
Let $x$ be a sample and $y$ its label.
We use PGD~\cite{M2017} to create adversaries and show how $B_{\F_y}(x)$ and $\F_y(x)$
change along with the steps of the adversarial training to explain the results in
Tables \ref{tab-mr1} and \ref{tab-cr1}.

In Figure \ref{fig-iar6},  we give the data of using the network
Lenet-5~\cite{LN5} for a sample $x$ with label $y$ in MNIST.
The blue,  orange,  green lines in the first picture
are Softmax$\F_y(x)$,   Softmax$B_{\F_y}(x)$,  Softmax$W_{\F_y}(x)$,  respectively.
The  blue,  orange,  green lines in the second picture
are $\F_y(x)$,   $B_{\F_y}(x)$,  $W_{\F_y}(x)$,  respectively.

When the blue line decreases,  we obtain  an adversary for $\F$,
which is not an adversary of $B_{\F}$,  because the orange line  does not reduce significantly.
For most samples from MNIST, the pictures are almost like this one,
and this explains why the values in lines 5-6 of Table \ref{tab-mr1} are low.

\begin{figure}[H]
\centering
\includegraphics[scale=0.3]{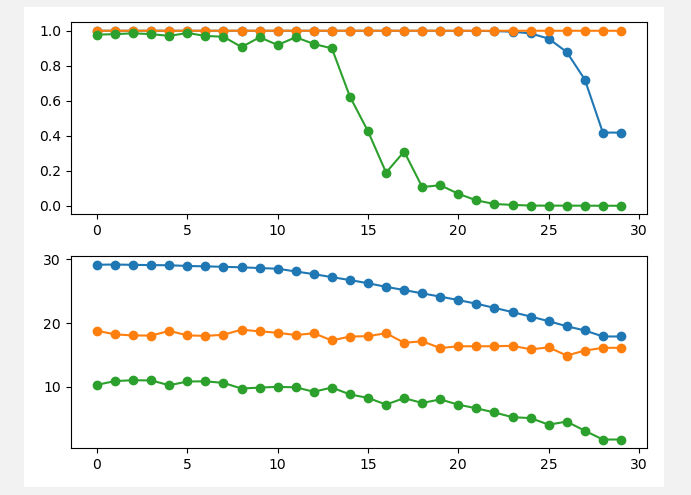}
\caption{Values of $\F_y(x)$,  $B_{\F_y}(x)$ and $W_{\F_y}(x)$  along with the adversarial training steps }
\label{fig-iar6}
\end{figure}

Similar results are given in in Figures  \ref{fig-iar8} for CIFAR-10 and network VGG-19~\cite{VGG}.
In this case,
the orange and the blue lines both decrease, and
the adversary of $\F$ is also an adversary of $B_{\F}$.
This explains why the values in lines 5-6 of Table \ref{tab-cr1} are higher than that of Table \ref{tab-mr1}.
%
%On the other hand, $W_{\F_y}(x)$ increases a lot.
%We will give a better loss function to solve this problem in section \ref{sec-b}

\begin{figure}[H]
\centering
\includegraphics[scale=0.3]{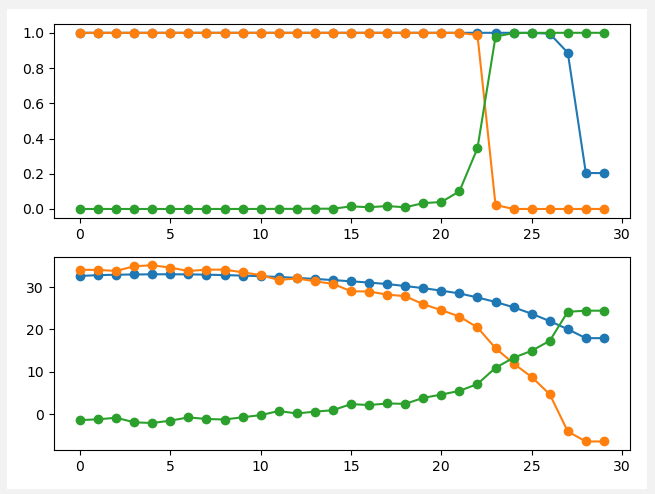}
\caption{Values of $\F_y(x)$,  $B_{\F_y}(x)$ and $W_{\F_y}(x)$  along with the adversarial training steps }
\label{fig-iar8}
\end{figure}
%The blue,  orange,  green lines in the first picture
%are Softmax$\F_y(x)$,   Softmax$B_{\F_y}(x)$,  Softmax$W_{\F_y}(x)$,  respectively.
%%
%The three blue,  orange,  green lines in the second picture
%are $\F_y(x)$,   $B_{\F_y}(x)$,  $W_{\F_y}(x)$,  respectively.

\subsection{Safety against original-model gradient-based  attack}
\label{sec-exp30}

In this section,  we use experimental results to validate
the results in Section \ref{sec-isafe}.
%Theorems \ref{th-safe1} and  \ref{th-safe3}.

\subsubsection{Rates of adversaries in random samples}
\label{sec-exp31}

In this section,  we give the   rates of a random
point near a sample  to be an adversary.
%
%which will be used to show the robustness of the bias classifier.
%
Two ways to select random points near a sample $x$ are used:
% (In there, $x$ has been given the correct label by the network.).

{\parskip=2pt\parindent=10pt%\small
R$_1$: Randomly change 60 pixels of $x$  from $b$ to $1-b$.

R$_2$: Add a random number in $[-0.2,0.2]$ to each pixel of $x$.
}

Two networks are used:
%Three kinds of networks are used for the data sets MNIST and CIFAR-10:

{\parskip=2pt\parindent=10pt%\small
$\NT_1$: Lenet-5 for MNIST and VGG-19 for CIFAR-10,  with normal training \eqref{eq-NT}.

$\NT_2$: Lenet-5 for MNIST and VGG-19 for CIFAR-10,  with adversarial training \eqref{eq-AT}.
}

In Table \ref{tab-randcr1},  we give the average   rates of adversaries.
From the table,  we can see that the rates for random samples to be adversaries are
quite low for networks $\NT_2$ and $B_{\F^{(2)}}$ trained with the adversarial training.

\begin{table}[H]
%\tiny
\center
\begin{tabular}{|c|c|c|c|c|c|c|}
  \hline
  % after \\: \hline or \cline{col1-col2} \cline{col3-col4} ...
  DNN & R$_1$/MNIST & R$_2$/MNIST &R$_1$/CIFAR-10 & R$_2$/CIFAR-10\\\hline
$\NT_1$ & 0.77$\%$ & 1.47$\%$ & 13.31$\%$ & 11.81$\%$ \\\hline
$\NT_2$ & 1.00$\%$ & 1.02$\%$ & 4.69$\%$ & 2.49$\%$ \\\hline
$B_{\F^{(2)}}$ & 0.85$\%$ & 1.64$\%$& 4.28$\%$ & 1.67$\%$\\\hline
\end{tabular}
\caption{Rates for random samples to be adversaries }
\label{tab-randcr1}
\end{table}

In Table \ref{tab-randcr101}, we give the values of $\Ca(\F,\rho)$ defined in \eqref{eq-sacr1}
for two values of $\rho$. We can see that $\Ca(\F,\rho)$ is a little bit smaller than
the values in Table \ref{tab-randcr1}, as expected.
\begin{table}[H]
%\tiny
\center
\begin{tabular}{|c|c|c|c|c|c|c|}
  \hline
  % after \\: \hline or \cline{col1-col2} \cline{col3-col4} ...
  DNN      & $\rho=0.1$/M     & $\rho=0.2$/M     &$\rho=0.1$/C        & $\rho=0.2$/C\\\hline
$\NT_1$        & 1.00$\%$         & 1.77$\%$         & 5.26$\%$           & 9.92$\%$ \\\hline
$\NT_2$        & 0.88$\%$         & 1.01$\%$         & 1.84$\%$           & $2.04\%$ \\\hline
$B_{\F^{(2)}}$ & 0.72$\%$         & 0.97$\%$         & 1.59$\%$           & 1.71$\%$\\\hline
\end{tabular}
\caption{$\Ca(\F,\rho)$. M means MNIST, C means CIFAR-10.}
\label{tab-randcr101}
\end{table}

\subsubsection{Safety of the bias classifier}
\label{sec-exp32}

For MNIST, let $\F^{(5)}=\F^{(2)}+W_5x$, where $\F^{(2)}$ is  given in   \eqref{eq-mod11} and $W_5\in\R^{10\times 784}$ is  from $\U_{10,784}(\lambda)$ for $\lambda=100$.

For CIFAR-10, let $\F^{(6)}=\F^{(2)}+W_6x$, where $\F^{(2)}$ is    in \eqref{eq-mod11} and $W_5\in\R^{10\times 3072}$ is from $\U_{10,3072}(\lambda)$ for $\lambda=100$.

The adversary creation rates   are given in Table  \ref{tab-gbr}, where the  adversaries are introduced in Section \ref{sec-oma1}.

%In there $3-i$ means use $FGSM$ to get the adversarial samples, change $L_2$ norm at most $i$.
%
%\begin{table}[H]
%\centering
%\begin{tabular}{|c|c|c|c|c|c|c|c|c|c|}
%  \hline
%  % after \\: \hline or \cline{col1-col2} \cline{col3-col4} ...
%  Network &        1-0.1              & 1-0.2            & 1-0.3      & 2-40        & 2-60         & 2-80 &3-3&3-4&3-5\\ \hline
%%  $\F(x)$ &        $98\%$             & $98\%$           & $99\%$       & $99\%$        & $99\%$         & 99$\%$ \\
%%
%%  $B_{\F}(x)$ &    $1\%$             & $2\%$           & $3\%$       & $2\%$        & $6\%$         & $8\%$ \\
% $B_{\F_A^{(7)}} $ for MNIST &    $1\%$             & $2\%$           & $2\%$       & $3\%$        & $3\%$         & $5\%$  &1$\%$   &1$\%$  &  2$\%$ \\
%  $B_{\F_A^{(8)}} $ for CIFAR-10&    $20\%$             & $20\%$           & $21\%$       & $22\%$        & $22\%$         & $24\%$    &21$\%$   &22$\%$  &  23$\%$ \\
%%
%  $B_{\F_B^{(7)}} $ for MNIST &    $1\%$             & $2\%$           & $3\%$       & $2\%$        & $3\%$         & $6\%$   &1$\%$   &2$\%$  &  2$\%$ \\
%  $B_{\F_B^{(8)}} $ for CIFAR-10&    $19\%$             & $20\%$           & $20\%$       & $21\%$        & $22\%$         & $23\%$    &21$\%$   &21$\%$  &  23$\%$ \\
%%
%%    $\F^{(1)}(x)$ & $3\%$ & 17$\%$ & 55$\%$ & $55\%$ & $79\%$ & $87\%$ \\
%%
%  \hline
%\end{tabular}
%\caption{Original-model gradient-based  attack for MNIST and CIFAR-10}
%\label{tab-gbr}
%\end{table}

\begin{table}[H]
%\tiny
\centering
\begin{tabular}{|c|c|c|c|c|c|c|}
  \hline
  % after \\: \hline or \cline{col1-col2} \cline{col3-col4} ...
  DNN &        1-1              & 1-2            & 1-3      & 2-40        & 2-60   &2-80     \\ \hline
%  $\F(x)$ &        $98\%$             & $98\%$           & $99\%$       & $99\%$        & $99\%$         & 99$\%$ \\
%
%  $B_{\F}(x)$ &    $1\%$             & $2\%$           & $3\%$       & $2\%$        & $6\%$         & $8\%$ \\
  $B_{\F^{(5)}} $ for MNIST &    $1\%$             & $2\%$           & $2\%$       & $2\%$        & $3\%$         & $4\%$  \\
  $B_{\F^{(6)}} $ for CIFAR-10&    $19\%$             & $20\%$           & $22\%$       & $21\%$        & $22\%$         & $24\%$    \\
%
 % $B_{\F_B^{(7)}} $ for MNIST &    $1\%$             & $2\%$           & $3\%$       & $2\%$        & $3\%$         & $6\%$   \\
 % $B_{\F_B^{(8)}} $ for CIFAR-10&    $19\%$             & $20\%$           & $20\%$       & $21\%$        & $22\%$         & $23\%$   \\
%
%    $\F^{(1)}(x)$ & $3\%$ & 17$\%$ & 55$\%$ & $55\%$ & $79\%$ & $87\%$ \\
%
  \hline
\end{tabular}
\caption{Original-model gradient-based  attack}
\label{tab-gbr}
\end{table}

From  Table  \ref{tab-gbr},
the bias classifier is safe against the original-model gradient-based  attack for MNIST,
and the adversarial creation rates in Table \ref{tab-gbr}
are close to those in Table \ref{tab-randcr1}.
%
%Furthermore, the results are  better than that given in Table \ref{tab-mr1},
%especially for the $l_{\infty}$ adversaries.

From  Table  \ref{tab-gbr}, the results are also near optimal for  CIFAR-10.
First,  comparing to the results  in Table \ref{tab-cr1},
the adversary creation rates are decreased  by half and are about $20\%$.
Second, from Table \ref{tab-41},  the accuracy of the bias classifier is about $82\%$.
%which means $\F$ gives wrong labels for about $18\%$ of the samples in the test set.
%The adversary creation rates in Table \ref{tab-gbr} are about $20\%$-$24\%$.
Comparing these data, the real adversary creation rates are about $1\%-6\%$
which are just above the  rates of random samples to be adversaries  in Table
\ref{tab-randcr1}.

\subsection{Black-box attack on the bias classifier}

In this section, we use the transfer-based black-box attack~\cite{zhuanyi}
to compare four networks: $\F^{(1)}$,  $\F^{(3)}$,  $\F^{(4)}$, $B_{\F^{(2)}}$ defined in Sections \ref{sec-exp1} and \ref{sec-oma1}.

The  black-box attack for  $\F$ works as follows.
A new network $\overline{\F}$ is trained with the training set
$\{(x, \F(x))\}$ for certain samples $x$.
%The structure of $\widehat{\F}$ is the same as $\F$ or the original model in case of bias classifier.
%
Then, we use PGD and JSMA to create adversaries for  $\overline{\F}$
and check wether they are adversaries of $\F$.
The adversary creation rates  are given in Table \ref{tab-bba1}.
We can see that, the bias classifier performs better for most
adversaries and in particular for $l_{\infty}$ adversaries.
Also, the adversary creation rates are about half of that of the original-model attack
in Tables \ref{tab-mr1} and \ref{tab-cr1}.
So the bias classifier has better robustness for the black-box attack in most cases.
\begin{table}[H]
\centering
{%\scriptsize
\begin{tabular}{|c|c|c|c|c|c|c|}
  \hline
  % after \\: \hline or \cline{col1-col2} \cline{col3-col4} ...
  DNN & 1-1 & 1-2 & 1-3 & 2-40 & 2-60 & 2-80 \\
  $\F^{(1)}$ & 1$\%$ & 2$\%$ & 18$\%$ & 28$\%$ & 35$\%$ & 40$\%$ \\
  $\F^{(3)}$ & 6$\%$ & 12$\%$ & 28$\%$ & 38$\%$ & 45$\%$ & 50$\%$ \\
  $\F^{(4)}$ & 1$\%$ & 3$\%$ & 21$\%$ & 24$\%$ & 39$\%$ & 46$\%$ \\
  $B_{\F^{(2)}}$ & 3$\%$ & 5$\%$ & 13$\%$ & 24$\%$ & 30$\%$ & 37$\%$ \\
 % $B_{\F^{(2)}}$ & 2$\%$ & 6$\%$ & 12$\%$ & 27$\%$ & 34$\%$ & 44$\%$ \\
%  $B_{\F^{(2)}}$ with loss function $\ref{eq-L3}$ & 2$\%$ & 4$\%$ & 10$\%$ & 28$\%$ & 30$\%$ & 39$\%$ \\
  \hline
\end{tabular}
}
\caption{Black-box attack of MNIST}
\label{tab-bba1}
\end{table}
\begin{table}[H]
\centering
{%\scriptsize
\begin{tabular}{|c|c|c|c|c|c|c|}
  \hline
  % after \\: \hline or \cline{col1-col2} \cline{col3-col4} ...
  DNN & 1-0.1 & 1-0.2 & 1-0.3 & 2-40 & 2-60 & 2-80 \\
  $\F^{(1)}$ & 22$\%$ & 23$\%$ & 28$\%$ & 35$\%$& 36$\%$ & 41$\%$ \\
  $\F^{(3)}$ & 27$\%$ & 29$\%$ & 36$\%$ & 40$\%$& 43$\%$ & 50$\%$ \\
  $\F^{(4)}$ & 21$\%$ & 24$\%$ & 28$\%$ & 33$\%$& 39$\%$ & 44$\%$ \\
  $B_{\F^{(2)}}$ & 21$\%$ & 23$\%$ & 24$\%$ & 33$\%$ & 36$\%$ & 41$\%$ \\
%  $B_{\F^{(2)}}$ & 20$\%$ & 24$\%$ & 25$\%$ & 29$\%$ & 34$\%$ & 43$\%$ \\
  \hline
\end{tabular}
}
\caption{Black-box attack of CIFAR-10}
\label{tab-bba2}
\end{table}

\subsection{Correlation attack}
\label{sec-catt}

In this section,  it is shown that the bias classifier is safe against the
correlation attack proposed in Section \ref{sec-safe4}.
The network used here is $\F^{(2)}$ given in \eqref{eq-mod11}
and the data set is CIFAR-10.
In Table \ref{tab-8},  we give the adversary creation rates
for samples which are given the correct label by $B_{\F^{(2)}}$.
%which are quite low comparing to that in
%Tables \ref{tab-cr1} and \ref{tab-bba2}.
Comparing to results in Tables \ref{tab-cr1} and \ref{tab-bba2},
%\ref{tab-gbr},
we can see that the bias classifier is quite safe against the correlation arrack.

\begin{table}[H]
\centering
%\small
\begin{tabular}{|c|c|c|c|c|c|c|}
  \hline
  % after \\: \hline or \cline{col1-col2} \cline{col3-col4} ...
  Network &        1-0.1              & 1-0.2            & 1-0.3      & 2-40        & 2-60         & 2-80 \\
  $B_{\F^{(2)}}$ &    $4\%$             & $11\%$           & $12\%$       & $8\%$        & $17\%$         & $21\%$ \\
  \hline
\end{tabular}
\caption{Adversary creation rates for   the correlation attack}
\label{tab-8}
\end{table}

In Figure \ref{fig-cra1}, we give the attack procedure.
It can be seen that when $(W_{x,y}-W_{x,i})x$ increases $B_{x,y}-B_{x,i}$ indeed decreases,
but $B_{x,y}-B_{x,i}$ does not decrease enough to change the label,
where $y$ is the label of $x$ and $i\ne y$.

\begin{figure}[H]
\centering
%\small
\includegraphics[scale=0.3]{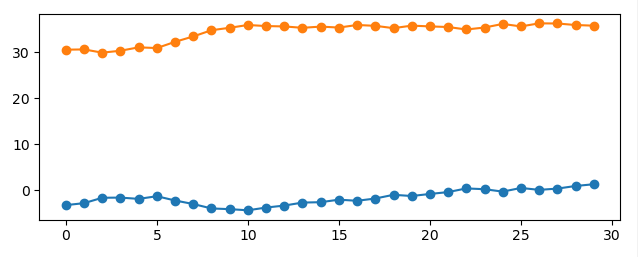}
\caption{
The input is an image for $9$ from MNIST.
The $x$-axis is the number of steps of the attack.
The blue line is $(W_{x,9}-W_{x,0})x$  and the orange line $B_{x,9}-B_{x,0}$.}
\label{fig-cra1}
\end{figure}

\subsection{Comparison with other methods }
In this section,  we compare our model with several existing models to defend adversaries. We use PGD-20 with $l_{\infty}$ bound $\epsilon=8/255$ to create adversaries on the test set of CIFAR-10.

In Table \ref{tab-com1},  we give the adversary creation rates for various attacks.
Our models are $\F^{(2)}$ in \eqref{eq-mod11} and $\F^{(6)}$ in Section \ref{sec-exp32}.
ResNet-10 \cite{Res} is used  in other cases.
The results for other networks are from the cited papers.
Gradient-based attacks  cannot be used for the bias classifier,
so we use the original-model attack given in Section \ref{sec-exp20}.
\begin{table}[H]%\small
\centering
\begin{tabular}{|l|c|}
  \hline
  % after \\: \hline or \cline{col1-col2} \cline{col3-col4} ...
 Attack Method & Adv. creation rates\\ \hline
%  ADV~\cite{M2017}      & 42.86$\%$ \\
%  TRADE~\cite{trades}   & 45.38$\%$ \\
%  MART                  & 46.60$\%$ \\
%  MMA~\cite{MMA}        & 37.26$\%$ \\
%  FOAR~\cite{Fo}        & 32.28$\%$ \\
%  SOAR~\cite{So}        & 56.06$\%$ \\
%  $\F^{(2)}$ in Sec. \ref{sec-exp1} & $58.90\%$ \\
  ADV~\cite{M2017}      & 57.1$\%$ \\
  TRADE~\cite{trades}   & 54.7$\%$ \\
%  MART                  & 53.4$\%$ \\
  MMA~\cite{MMA}        & 62.7$\%$ \\
  FOAR~\cite{Fo}        & 67.7$\%$ \\
  SOAR~\cite{So}        & 44.0$\%$ \\
  $B_{\F^{(2)}}$ in Sec. \ref{sec-exp1} & 41.1$\%$ \\
  $B_{\F^{(6)}}$ in Sec. \ref{sec-exp32} &     $20\%$  \\
  \hline
\end{tabular}
\caption{Adversary creation rates  for CIFAR-10}
\label{tab-com1}
\end{table}

Although the DNN models and the attacks are not the same,
this comparison gives a rough idea of the performance that can be achieved for various
methods of defending adversaries.
%
%since the accuracy is about $82\%$ for this dataset.

From Table \ref{tab-com1}, we see that
the attack method SOAR~\cite{So} and the bias classifier $B_{\F_B^{(2)}}$ achieve the best
results for creating lower rates of adversaries, besides $B_{\F_B^{(6)}}$.
As explained in Section \ref{sec-exp32},
the optimal adversary creation rate  is about $20\%$
and is achieved by $B_{\F_B^{(6)}}$.
%As we explained in Section \ref{sec-exp32}, $20\%$ is near the optimal results.

\subsection{Summary of the experiments}
We give a summary of the experiments in this section.

From Tables \ref{tab-1} and \ref{tab-41}, we can see that
the bias classifier achieves comparable accuracies
with DNNs of similar sizes.

From Table \ref{tab-randcr1}, we can see that
the bias classifier with a random first-degree part is safe against
gradient-based attacks, as proved in Section \ref{sec-isafe}.

From Tables \ref{tab-1}, \ref{tab-mr1}, \ref{tab-cr1}, \ref{tab-bba1}, \ref{tab-bba2}, and \ref{tab-com1},
we can see that the bias classifier is more robust than DNNs with similar sizes against adversarial attacks .

From Tables \ref{tab-cr1}, \ref{tab-bba2}, \ref{tab-8},
the original-model attack, the black-box attack, and the correlation attack
become weaker for creating adversaries, and the original model attack is the best available attack for
the bias classifier.

\section{Concluding remarks}
\label{sec-conc}

In this paper, we show that the bias part of a DNN can be effectively trained as
a classifier.
The motivation to use the bias part as the classifier is that
gradients of the DNN seems to be inevitable to generate adversaries efficiently and
the bias part of a DNN with Relu as activation functions
is a piecewise constant function with zero gradient
and is safe against direct gradient-based attacks such as FGSM.

The bias classifier can be effectively trained with
the adversarial training method~\cite{M2017},
which increases the classification power of the bias part
and  decreases the classification power of first-degree part.
Experimental results are used to show the robustness of the bias classifier
over the standard DNNs

Further, by adding a random first-degree part to the bias classifier,
an information-theoretically safe classifier against gradient-based attacks is obtained,
that is, the adversary creation rate
is almost the same as the rate of certain random samples to be adversaries.

For further research, the estimations in Theorems \ref{th-safe6}, \ref{th-safe4}, \ref{th-safe5}
are not optimal, and better estimations are desirable.

\newpage

\section*{Appendix}
\subsection*{Appendix A. Proof of Theorem \ref{th-E1}}
\setcounter{section}{2}
\gtthmE*
We first prove several lemmas. In this section,
the notations $\O,  O_l,  \L$ introduced in Theorem \ref{th-E1} will be used.

Let $\Gamma:\R\rightarrow\R$ be another activation function:
$$\Gamma(x)=
%
%=0 \hbox{ if } x\le 0$,   and $\alpha(x)=1$ if $x>0$.
%
\left\{
  \begin{array}{lll}
    0 & \hbox{ if }&  x\le 0 \\
    1 & \hbox{ if }& x>0 \\
  \end{array}
\right.
.$$

\begin{lemma}[Theorem 5 in \cite{G1989}]
\label{zl-1}
Let $l\in\L$ and $F_l:\R\to\R$ be a function such that $F_l(x)=1$ if $x\in O_l$
and $F_l(x)=-1$ otherwise.
Then for any $\epsilon>0$,   there exist $N\in \N_{>0}$,   $W\in \R^{N\times n}$,   $b\in \R^{N}$,   $U\in \R^{1\times N}$,   and an open set $D\subset  \I^n$ with $V(D)<\epsilon$,
such that
$$G(x)=U\cdot \Gamma(W x+b):\I^n\to\R$$
and $|G(x)-F_l(x)|<\epsilon$ for $x\in \O\setminus D$.
\end{lemma}

The following lemma shows that there exists a DNN with one hidden layer
and using $\Gamma$ as the activation function,  which can be used as a classifier
for $\O$.
\begin{lemma}
\label{z-1}
For any $\epsilon>0$,   there exist $N\in \N_{>0}$,   $W\in \R^{N\times n}$,   $b\in \R^{N}$,   $U\in \R^{m\times N}$, and an open set $D\subset  \I^n$ with $V(D)<\epsilon$,
such that
$$\G(x)=U \cdot \Gamma(W x+b):\I^n\to \R^m$$
gives the correct label for $x\in \O\setminus D$.
\end{lemma}
\begin{proof}
By Lemma \ref{zl-1},   for $l\in\L=[m]$,  there exist
$N_a\in \N_{>0}$,
$W_l\in \R^{N_a\times  n}$,
$b_l\in \R^{N_a}$,
$U_l \in \R^{1\times  N_a}$,
and $D_l\subset \I^n$ with $V(D_l)<\epsilon/m$
such that
$$G_l(x)=U_l  \cdot \Gamma(W_l x+b_l)
\hbox{ and }
|G_l(x)-F_l(x)|<\epsilon$$
for $x\in \O\setminus D_l$, where $F_l$ is defined in Lemma \ref{zl-1}.
%
%It is suitable to make $N_l=N_w\in \N^+$ for all $l$.

Let $N=N_a m$, $W\in \R^{N\times n}$,   $b\in \R^{N}$,
where the $l$-th row of $W$ is the $l_2$-th row of $W_{l_1}$ and
the $l$-th row of $b$ is the $l_2$-th row of $b_{l_1}$,
where  $l=l_1 N_a+l_2$,   $0\le l_2< N_a$, and $0\le l_1< m$.

Let $U\in \R^{m\times  N}$ be formed as follows:
for $j\in [m]$,   the $j$-th row of $U$ are zeros except the
$(j-1)N_a$-th to the $((j-1)N_a+N_a-1)$-th rows,
and the values of the $((j-1)N_a+k)$-th place of the $j$-th row of
$U$ equal to the values of the $k$-th place of $U_l$,  where $k=0,1,\ldots,N_a-1$.

Now we have $N=N_am\in \N_{>0}$,   $W\in \R^{N\times  n}$,   $b\in \R^{N}$,   $U\in \R^{m\times  N}$, and
$$\G(x)=U\cdot  \Gamma(W x+b)$$
satisfies  $\G(x)_l$=$G_l(x)$,   where $\G(x)_l$ is the $l$-th coordinate of $\G(x)$.
Let $D=\bigcup_{i=1}^m D_l\subset \I^n$ with $V(D)<\epsilon$.
Then  $\G(x)_y>1-\epsilon$ and $\G(x)_l<-1+\epsilon$ for $x\in O_y$ and $l\ne y$.
%where $y$ is the label of $x$.
Since $\epsilon$ can be as small as possible,
we have $\G(x)_y>\G(x)_l$ for $l\ne y$,
and  $\G(x)$ give label $y$ for $x\in O_y\setminus D$.
Hence, $\G(x)$ gives the correct label for $x\in \O\setminus D$.
\end{proof}

\begin{lemma}
\label{zbc-2}

Let $W\in\R^{1\times n}$ have nonzero entries and $b\in\R$.
For any $a>0$, let $Z_a=\{x\in \I^n\,|\,|Wx+b|<a\}$.
Then  $V(Z_a)\le2a\sqrt{n}^{n-1}/||W||_2$.
%$m(Z_a)$ is the volume of $Z_a$.
\end{lemma}
\begin{proof}
Let $U=\{U_1,\ldots,U_n\}$ be a unit orthogonal basis of $\R^n$ and $U_1=\frac{W}{||W||_2}$.
If $T_i=\max_{x,y\in Z_a}\{\langle x-y,U_i\rangle\}$, then we have $V(Z_a)\le\Pi_{i=1}^{n}T_i$.

For $i>1$, we have $T_i\le \max_{x,y\in Z_a}||x-y||_2\le \sqrt{n}$.
Moreover, for any $x,y\in Z_a$,  we have
\begin{equation*}
\begin{array}{ll}
&\langle x-y,U_1\rangle\\
=&\langle x-y,W\rangle/||W||_2\\
=&(Wx+b-Wy-b)/||W||_2\\
\le& 2a/||W||_2\\
\end{array}
\end{equation*}
which means $T_1\le 2a/||W||_2$. Then we have
$$V(Z_a)\le\Pi_{i=1}^{n}T_i\le2a\sqrt{n}^{n-1}/||W||_2.$$
The lemma is proved.
\end{proof}

\begin{lemma}
\label{zl-2}
The bias vector $b$  in Lemma \ref{z-1} can be chosen to consist of nonzero values.
\end{lemma}
\begin{proof}
By Lemma \ref{z-1},   there exist
$N\in \N_{>0}$,   $W\in \R^{N\times  n}$,   $b\in \R^{N}$,   $U\in \R^{m\times  N}$,
 and $D_1\subset \I^n$ with $V(D_1)<\epsilon/2$,   such that
$$\G(x)=U\cdot \Gamma(W x+b)$$
gives the correct label for $x\in \O\setminus D_1$.

Let $\gamma=\frac{\epsilon W_m}{4N \sqrt{n}^{n-1}}$, where $W_m=||W||_{2,\infty}$.
Assume $\widetilde{b}=b-I_0(|b|) \gamma$, where $I_0(x)=1-\sign(x)$ and
when $I_0$ is treated as a map of a vector,
it acts on each entry of the vector, respectively.
From the construction, $\widetilde{b}$ does not have zero entries, because $\widetilde{b}_i=b_i$ if $b_i\ne0$, and $\widetilde{b}_i=\gamma$ if $b_i=0$,
where $b_i$ and $\widetilde{b}_i$ are respectively the $i$-th rows of $b$ and $\widetilde{b}$.

Let $W_i$ be the $i$-th row of $W$ and  $Z_i=\{z\in\R^n\,|\,|W_i z+b_i|<\gamma\}$.
By Lemma \ref{zbc-2}, we have $V(Z_i\bigcap \I^n)<2\gamma\sqrt{n}^{n-1}/W_m$. We write $C_n=\sqrt{n}^{n-1}/W_m$.

Let $Z=\{x\in\R^n\,|\,\Gamma(W x+b)\ne\Gamma(W x+\widetilde{b})\}$.
We will show that  $Z=\cup_{i=1}^N Z_i$.
If $\Gamma(W x+b)\ne\Gamma(W x+\widetilde{b})$,
then there exists an $i\in [N]$ such that
$W_i x+b_i>0$ and $W_i x+\widetilde{b}_i<0$,
or $W_i x+b_i<0$ and $W_i x+\widetilde{b}_i>0$.
%,   where $W_i$ is the $i$-th row of $W$,
%$b_i$ and $\widetilde{b}_i$ are respectively the $i$-th rows of $b$ and $\widetilde{b}$.
%
If $W_i x+b_i>0$ and $W_i x+\widetilde{b}_i<0$,
then
$W_i x+\widetilde{b}_i = W_i x+b_i -I_0(|b|)\gamma<0$
and hence $|W_i x+ b_i|\le \gamma$.
Similarly, if $W_i x+b_i<0$ and $W_i x+\widetilde{b}_i>0$,
we also have $|W_i x+ b_i|\le \gamma$,
which implies $x\in Z_i$.
As a consequence   $Z=\cup_{i=1}^N Z_i$.

From $Z=\cup_{i=1}^N Z_i$, we have $V(Z\bigcap \I^n)<2\gamma N C_n< \epsilon/2$,
since  $\gamma=\frac{\epsilon}{4N C_n}$.
Let $D=D_1\bigcup(Z\bigcap \I^n)\subset  \I^n$.
Then $V(D)<V(D_1)+V(Z\bigcap \I^n)<\epsilon$.

Finally, let
$$\widetilde{\G}(x)=U\cdot \Gamma (W x+\widetilde{b}).$$
Then, for $x\in\O\setminus D$, we have $\Gamma(W x+b)=\Gamma(W x+\widetilde{b})$
and hence $\widetilde{\G}(x) = {\G}(x)$.
That is, $\widetilde{\G}$ satisfies the conditions of the lemma.
%
%The lemma is proved.
\end{proof}

\begin{lemma}
\label{z-2}
Let $\G:\I^n\to\R^m$ be a one-hidden-layer DNN with activation function $\Gamma(x)$,
and any coordinate of its bias vector is nonzero.
Then there exists a DNN $\F$,   which has the same structure as $\G$,
except that the activation function of $\F$ is \Relu,   such that
$B_{\F}(x)=\G(x)$
for all $x\in\I^n$.
\end{lemma}
\begin{proof}
Assume $\G(x)=U \cdot \Gamma(W x+b)+c$.
Let  $\F(x)=U^{\F} \Relu(W x+b)+c$,   where
$U^{\F}=U\diag(\frac{1}{b_i})$ and $b_i$ is the $i$-th entry of $b$.
We will show that  $\F$ satisfies the condition of the lemma.
By the definition of $\Gamma$, the constant part of
$\Relu(W x+b)$ is $b\circ \Gamma(W x+b)$, where $\circ$ is the point-wise product.
So, $B_{\F}(x)=U^{\F}(b\circ \Gamma(W x+b))+c = U\diag(\frac{1}{b_i})(b\circ \Gamma(W x+b))+c =U\Gamma(W X+b)+c$.
%Then We Have
%\Begin{Equation*}
%%\Scriptsize
%\Small
%\Begin{Array}{Ll}
%&U^{\F}(B^\F\Circ\Gamma(W^\F X+B^\F))\\
%=&U(\Diag(\Frac{B_I}{|B_I|}))(B^\F\Circ\Gamma(W X+B))\\
%=&U(|B^\F|\Gamma(W X+B))\\
%=&U\Gamma(W X+B).\\
%\End{Array}
%\End{Equation*}
%Hence,
$B_{\F}(x)=\G(x)$ and the lemma is proved.
\end{proof}

{\em Proof of Theorem \ref{th-E1}.}
By Lemma \ref{z-1},
there exist a $D\subset  \I^n$ with $V(D)<\epsilon$
and a network $\G$ with one-hidden-layer and with activation function $\Gamma(x)$,
such that $\G(x)$ gives the correct label for $x\in \O\setminus D$.
By Lemma \ref{zl-2}, all the parameters of $\G$ are nonzero.
Then by Lemma \ref{z-2},  we can obtain a network $\F$ with Relu as the activation function
such that $B_{\F}=\G(x)$,
and the theorem is proved.\qed

\subsection*{Appendix B.  Proof of Theorem \ref{th-safe6}}
\setcounter{section}{4}
\setcounter{lemma}{0}
We first prove two lemmas.
\begin{lemma}
\label{ls-sa}
Let $\{u_i\}_{i=1}^{n}$ be a set of iid random variables with values in $[-\lambda,\lambda]$ and $u=\sum_{i=1}^nx_iu_i$, where $x_i\in\R$ such that  $|x_i|>a>0$ for some $i$.
Let the density function of $u$ be $f(x)$.
Then $f(x)<\frac{1}{2\lambda a}$ for all $x$.
\end{lemma}
\begin{proof}
Assume $|x_n|>a$ and $f_n(x)$ is the distribution function of $x_nu_n$. We have
\begin{equation*}
%\scriptsize
\begin{array}{ll}
&P(u<m)\\
=&\displaystyle\int_{\{-\lambda|x_i|\}_{i=1}^{n-1}}^{\{\lambda|x_i|\}_{i=1}^{n-1}}(\Pi_{i=1}^{n-1}\frac{1}{2\lambda|x_i|})f_n(m-\sum_{i=1}^{n-1}t_i)\d t_1t_2\dots t_{n-1}.
\end{array}
\end{equation*}

Since $0<f'_n(x)\le\frac{1}{2\lambda|x_n|}$ and $f(x)=\frac{\nabla P(u<x)}{\nabla x}$, we have
\begin{equation*}
\renewcommand{\arraystretch}{1.5}
%\small
\begin{array}{ll}
&f(x)\\
=&\frac{\nabla P(u<x)}{\nabla x}\\
=&\frac{\nabla \dint_{\{-\lambda|x_i|\}_{i=1}^{n-1}}^{\{\lambda|x_i|\}_{i=1}^{n-1}}
(\Pi_{i=1}^{n-1}\frac{1}{2\lambda|x_i|})f_n(x-\sum_{i=1}^{n-1}t_i)\d t_1t_2\dots t_{n-1} }{\nabla x}\\
=&\dint_{\{-\lambda|x_i|\}_{i=1}^{n-1}}^{\{\lambda|x_i|\}_{i=1}^{n-1}}(\Pi_{i=1}^{n-1}\frac{1}{2\lambda|x_i|})\frac{\nabla f_n(x-\sum_{i=1}^{n-1}t_i)}{\nabla x}\d t_1t_2\dots t_{n-1}\\
\le&\dint_{\{-\lambda|x_i|\}_{i=1}^{n-1}}^{\{\lambda|x_i|\}_{i=1}^{n-1}}(\Pi_{i=1}^{n-1}\frac{1}{2\lambda|x_i|})\frac{1}{2\lambda|x_n|}\d t_1t_2\dots t_{n-1}\\
\le&\frac{1}{2\lambda|x_n|}\\
\le&\frac{1}{2\lambda a}.
\end{array}
\end{equation*}
The lemma is proved.
\end{proof}

\begin{lemma}
\label{ls-sabc}
Let $\{u_i\}_{i=1}^{n}$ be a set of iid variables, $f_i$ the density function of $u_i$, and %there exists a $k$ such that
$f_i(x)<a$ for all $x\in\R$.
Then we have
$$P(|u_i-u_j|>\psi\ \hbox{for}\ \forall i\ne j)>\Pi_{i=0}^{n-1}(1-2i\psi a).$$
\end{lemma}
\begin{proof}
Let $D_k$ be the event $|u_i-u_j|>\psi$ for $\forall i,j\le k$, and $F_{k}:\R^k\to \R$   the joint probability density function of $\{u_i\}_{i=1}^{k}$ under condition $D_k$.
Then we have
\begin{equation*}
\renewcommand{\arraystretch}{1.5}
%\tiny
\begin{array}{ll}
&P(D_k)\\
=&P(D_k,D_{k-1})\\
=&P(D_k\|D_{k-1})P(D_{k-1})\\
=&P(|u_k-u_i|>\psi\ for\ \forall i< k\|D_{k-1})P(D_{k-1})\\
=&P(D_{k-1})\int_{-\infty^{k-1}}^{\infty^{k-1}}\int_{-\infty}^{\infty}F_{k-1}(t_1,\dots,t_{k-1})\\
&f_k(t_k)I(|t_k-t_i|>\psi\ \forall i< k)\d t_k\d t_1\dots t_{k-1}\\
>&P(D_{k-1})\int_{-\infty^{k-1}}^{\infty^{k-1}}\int_{-\infty}^{\infty}F_{k-1}(t_1,\dots,t_{k-1})\\
&(f_k(t_k)-aI(|t_k-t_i|<\psi\ \exists i< k))\d t_k\d t_1\dots t_{k-1}\\
=&P(D_{k-1})(1-\int_{-\infty^{k-1}}^{\infty^{k-1}}\int_{-\infty}^{\infty}aF_{k-1}(t_1,\dots,t_{k-1})\\
&I(|t_k-t_i|<\psi\ \exists i< k)\d t_k\d t_1\dots t_{k-1})\\
>&P(D_{k-1})(1-\\
&\int_{-\infty^{k-1}}^{\infty^{k-1}}
2a(k-1)\psi F_{k-1}(t_1,\dots,t_{k-1})\d t_1\dots t_{k-1})\\
=&P(D_{k-1})(1-2a(k-1)\psi).
\end{array}
\end{equation*}
Since $P(D_0)=1$, we have
\begin{equation*}
\renewcommand{\arraystretch}{1.5}
%%\small
%\scriptsize
\begin{array}{ll}
&P(|u_i-u_j|>\psi\ for\ \forall i\ne j)\\
=&P(D_n)\\
>&P(D_{n-1})(1-2(n-1)\psi a)\\
>&P(D_{n-2})(1-2(n-1)\psi a)(1-2(n-2)\psi a)\\
>&\dots\\
>&\Pi_{i=0}^{n-1}(1-2i\psi a).
\end{array}
\end{equation*}
The lemma is proved.
\end{proof}
%\begin{theorem}
%Assume $|\frac{\nabla \F(x)}{\nabla x}|_{\infty}<\mu/2$,
% $|B_{\F}(x)|_{\infty}<\beta$,
% $P_{x\sim D_{\O}}(|x|_{\infty}<\eta)<\eta_1$,
%  and $(\lambda-\mu)e^{-2\beta+\sqrt{\lambda}}>(2m\lambda+\mu)m$.
%If $W_R\sim \M_{m,n}(\lambda)$, then
%$\Ca(B_{\widetilde{\F}},\A_2,\M_{m,n}(\lambda))\le (m-1)\Ca(\F,\rho)+\frac{(m-1)(m-2)}{2\sqrt{\lambda} \eta}+\eta_1$.
%\end{theorem}
%
\gtthmo*
%We now prove  Theorem \ref{th-safe6}.\\
%{\em Proof of Theorem \ref{th-safe6}.}
\begin{proof}
From \eqref{eq-dnn1} and \eqref{eq-gnet1}, we have
$\F(x) = W_x x + B_x$ and $\widetilde{\F}(x) = (W_x+W_R) x + B_x$.
Let $x$ be a sample with label $y$.
From equation \eqref{eq-fg11}, we have
\begin{equation*}
\begin{array}{ll}
&\frac{\nabla L(\widetilde{\F}(x),y)}{\nabla x}
=\frac{\sum_{i=1}^{m}(W_{R,i}-W_{R,y}+W_{x,i}-W_{x,y})e^{\widetilde{\F}_i(x)}}{\sum_{i=1}^{m}e^{\widetilde{\F}_i(x)}}.
\end{array}
\end{equation*}
Let $m_x =\arg\max_{i\ne y}\{\langle W_{R,i},x \rangle\}$ and consider the condition:

{\bf Condition $C_1$}: $\langle W_{R,m_x },x \rangle\ >\ \langle W_{R,j},x \rangle+\sqrt{\lambda}$ for all $j\in[m]\setminus\{y,m_x\}$.
%(2): $|W_{R,m_x}-W_{R,y}|_i>\mu$ for $i\in[n]$.\\

We first give the  probability for condition $C_1$ to be valid.
By Lemmas \ref{ls-sa} and \ref{ls-sabc} and due to $|x|_\infty=1$, we have
\begin{equation}
\label{eq-th431}
\renewcommand{\arraystretch}{1.6}
%\tiny
%\scriptsize
\begin{array}{ll}
&P_{W_R\sim \M_{m,n}(\lambda)}(C_1)\\
\ge& P_{W_R\sim \M_{m,n}(\lambda)}(|\langle W_{R,j},x \rangle-\langle W_{R,i},x \rangle|>\sqrt{\lambda},\\
&\forall i,j\in[m]/\{y\},\ i\ne j)\\
\ge&\Pi_{i=1}^{m-2}(1-\frac{2i\sqrt{\lambda}}{2\lambda|x|_\infty})\\
\ge&(1-\frac{m-2}{\sqrt{\lambda}})^{m-2}\\
\ge&1-\frac{(m-2)^2}{\sqrt{\lambda}}.
\end{array}
\end{equation}
Let $||x||_{-\infty}=\min_{i\in[n]}\{|x|_i\}$ for $x\in\R^n$.
Since $|\frac{\nabla \F(x)}{\nabla x}|_{\infty}<\mu/2$ and $W_R\sim \M_{m,n}(\lambda)$, we have
$||W_{R,i}+W_{R,j}||_{-\infty}>\lambda$, $||W_{R,i}+W_{R,j}||_{\infty}< 2m\lambda$
and $||W_{x,i}+W_{x,j}||_{\infty} < \mu$ for any $i\ne j$.
If condition $C_1$ is satisfied,
then for any $j\in[m]\setminus\{y,m_x\}$, we have%
\begin{equation*}
\renewcommand{\arraystretch}{1.6}
%\scriptsize
\begin{array}{ll}
 &\widetilde{\F}_{m_x}(x)-\widetilde{\F}_{j}(x)\\
=&(W_{R,m_x}+W_{x,m_x}-W_{R,j}-W_{x,j})x+B_{x,m_x}-B_{x,j}\\
=&(W_{R,m_x}-W_{R,j})x+(W_{x,m_x}-W_{x,j})x+B_{x,m_x}-B_{x,j}\\
>&\sqrt{\lambda}-n\mu-2\beta.
\end{array}
\end{equation*}
Further considering the hypothesis $(\lambda-\mu)e^{-2\beta-n\mu+\sqrt{\lambda}}>(2m\lambda+\mu)m$,
 we have
\begin{equation*}
\renewcommand{\arraystretch}{1.6}
%\scriptsize
\begin{array}{ll}
 &||W_{R,m_x }-W_{R,y}+W_{x,m_x }-W_{x,y}||_{-\infty} e^{\widetilde{\F}_{m_x }(x)}\\
 >&(\lambda-\mu)e^{\widetilde{\F}_{m_x }(x)}\\
 >&(\lambda-\mu)e^{{\widetilde{\F}_{j}(x)}+\sqrt{\lambda}-2\beta-n\mu}\\
=&(\lambda-\mu)e^{-2\beta-n\mu}e^{\sqrt{\lambda}}e^{{\widetilde{\F}_{j}(x)}}\\
>&(2m\lambda+\mu)me^{{\widetilde{\F}_{j}(x)}}\\
>&m||(W_{R,j}-W_{R,y}+W_{x,j}-W_{x,y})||_\infty e^{\widetilde{\F}_{j}(x)}
\end{array}
\end{equation*}
which means
\begin{equation*}
%\scriptsize
\begin{array}{ll}
 \sign(\sum_{i=1}^{m}(W_{R,i}-W_{R,y}+W_{x,i}-W_{x,y})e^{\widetilde{\F}_i(x)})=\\
\sign((W_{R,m_x }-W_{R,y}+W_{x,m_x }-W_{x,y})e^{\widetilde{\F}_{m_x }(x)}).
\end{array}
\end{equation*}
Because of this, we have:
\begin{equation*}
\renewcommand{\arraystretch}{1.6}
%%\small
%\scriptsize
\begin{array}{ll}
&\sign(\frac{\nabla L(\widetilde{\F}(x),y)}{\nabla x})\\
=&\sign(\frac{\sum_{i=1}^{m}(W_{R,i}-W_{R,y}+W_{x,i}-W_{x,y})e^{\widetilde{\F}_i(x)}}{\sum_{i=1}^{m}e^{\widetilde{\F}_i(x)}})\\
=&\sign(\sum_{i=1}^{m}(W_{R,i}-W_{R,y}+W_{x,i}-W_{x,y})e^{\widetilde{\F}_i(x)})\\
=&\sign((W_{R,m_x }-W_{R,y}+W_{x,m_x }-W_{x,y})e^{\widetilde{\F}_{m_x }(x)})\\
=&\sign((W_{R,m_x }-W_{R,y}+W_{x,m_x }-W_{x,y})\\
=&\sign(W_{R,m_x }-W_{R,y}).
\end{array}
\end{equation*}

Let $V$ be a random vector in $\{0,1\}^n$.
Then the probability for the sign of $W_{R,m_x }-W_{R,y}$
to be $V$ is
\begin{equation*}
\renewcommand{\arraystretch}{1.5}
%\small
\begin{array}{ll}
&P(\sign(W_{R,m_x }-W_{R,y})=V,C_1)\\
\le& P(\sign(W_{R,m_x }-W_{R,y})=V)\\
=&\sum_{i<y}P(m_x=i,   \sign(W_{R,y})=V)+\\
&\sum_{i>y}P(m_x=i,   \sign(W_{R,i})=V)\\
\le& \sum_{i<y}P(\sign(W_{R,y})=V)\\
&\sum_{i>y}P(\sign(W_{R,i})=V)\\
=&\frac{m-1}{2^n}.
\end{array}
\end{equation*}
So  we have
\begin{equation*}
\renewcommand{\arraystretch}{1.6}
%\tiny
%%\scriptsize
\begin{array}{ll}
&\E_{W_R\sim \M_{m,n}(\lambda)}\\
 &[\ID(\widehat{B}_{\F}(x+\rho\sign(\frac{\nabla L(\widetilde{\F}(x),y)}{\nabla x}))\ne\widehat{B}_{\F}(x))\ID(C_1)]\\
=&\E_{W_R\sim \M_{m,n}(\lambda)}\\
 &[\ID(\widehat{B}_{\F}(x+\rho\sign(W_{R,m_x }-W_{R,y}))\ne\widehat{B}_{\F}(x))\ID(C_1)]\\
=&\sum_{V\in\{-1,1\}^{n}}P(\sign(W_{R,m_x }-W_{R,y})=V,\ C_1)\\
&\ID(\widehat{B}_{\F}(x+\rho V)\ne\widehat{B}_{\F}(x))\\
\le& \sum_{V\in\{-1,1\}^{n}}(m-1)/(2^n)\ID(\widehat{B}_{\F}(x+\rho V)\ne\widehat{B}_{\F}(x))\\
=&(m-1)\Ca(\F,\rho).
\end{array}
\end{equation*}
Finally, from \eqref{eq-th431} we have
\begin{equation*}
\renewcommand{\arraystretch}{1.6}
%\tiny
%%\scriptsize
\begin{array}{ll}
&\Ca(B_{\widetilde{\F}}\A_2,\M_{m,n}(\lambda))\\
=&\E_{x\sim D_{\O}}\E_{W_R\sim \M_{m,n}(\lambda)}\\
 &[\ID(\widehat{B}_{\F}(x+\rho  \sign(\frac{\nabla L(\widetilde{\F}(x),y)}{\nabla x}))\ne\widehat{B}_{\F}(x))]\\
\le&\E_{x\sim D_{\O}}\E_{W_R\sim \M_{m,n}(\lambda)}\\
&[\ID(\widehat{B}_{\F}(x+\rho  \sign(\frac{\nabla L(\widetilde{\F}(x),y)}{\nabla x}))\ne\widehat{B}_{\F}(x))\ID(C_1)\\&+(1-\ID(C_1))]\\
\le& (m-1)\Ca(\F,\rho)+\E_{x\sim D_{\O}}\E_{W_R\sim \M_{m,n}(\lambda)} [(1-\ID(C_1))]\\
\le& (m-1)\Ca(\F,\rho)+\E_{x\sim D_{\O}}[1-P_{W_R\sim \M_{m,n}(\lambda)}(C_1)]\\
%\le& (m-1)\Ca(\F,\rho)+\frac{(m-1)(m-2)}{2\sqrt{\lambda} \eta}+P_{x\sim D_{\O}}(|x|_{\infty}\le \eta)\\
\le& (m-1)\Ca(\F,\rho)+\frac{(m-2)^2}{\sqrt{\lambda}}.\\
\end{array}
\end{equation*}
%{\color{red}
%We can assume $|x|_\infty=1$, so the condition
%$P_{x\sim D_{\O}}(|x|_{\infty}<\eta)<\eta_1$ is not needed.
%
%And the result becomes
%$\Ca(B_{\widetilde{\F}},\A_2,\M_{m,n}(\lambda))\le (m-1)\Ca(\F,\rho)+\frac{(m-1)(m-2)}{2\sqrt{\lambda}}$.
%?
%}
The theorem is proved.
\end{proof}
%The theorem is proved. \qed

\subsection*{Appendix C.  Proof of Theorem \ref{th-safe3}}
We first prove a lemma.
\begin{lemma}
\label{lm-cha}
Let $x_1, x_2\sim \U(-\lambda ,\lambda )$  and $z=x_1-x_2$.
Then for $a\in[0,2\lambda ]$, we have $P(z<a)=P(z>-a)=1-\frac{(2\lambda -a)^2}{8\lambda ^2}$,
which is denoted as $T(\lambda ,a)=1-\frac{(2\lambda -a)^2}{8\lambda ^2}$.
\end{lemma}
\begin{proof}
Let $f(z)$ be the density function of $z$.
Then $f(z)=0$, if $z\ge 2\lambda $ or $z\le -2\lambda $;
$f(z)=\frac{2\lambda +z}{4\lambda ^2}$, if $0\ge z\ge -2\lambda $;
$f(z)=\frac{2\lambda -z}{4\lambda ^2}$, if $0\le z\le 2\lambda $.
Hence,  $P(z<a)=P(z>-a)=1-\frac{(2\lambda -a)^2}{8\lambda ^2}$.
\end{proof}
Note that  $T(\lambda ,a)$   increases with $a$ and $T(\lambda ,a)\in[0.5,1]$.

%\begin{theorem}
%\label{th-safe13}
%If $|\frac{\nabla \F(x)}{\nabla x}|_{\infty}<\lambda/2$, then
%$Sa_\H(\lambda)<\Ca(\rho)+na/\lambda $.
%If  $\lambda >na/\epsilon$, then $Sa_\H(\lambda)<\Ca(\rho)+\epsilon$.
%\end{theorem}

%We give the proof of Theorem \ref{th-safe3}.\\
%{\em Proof of Theorem \ref{th-safe3}.}
\gtthmS*
\begin{proof}
%Let $||x||_{-\infty}=\min_{i\in[n]}\{|x|_i\}$ for $x\in\R^n$.
Similar to \eqref{eq-sf11},
if $||W_{R,n_x}-W_{R,y}||_{-\infty}>\mu$, then we have
\begin{equation*}
\renewcommand{\arraystretch}{1.6}
%\scriptsize
\begin{array}{ll}
&\A_3(x,\widetilde{\F},\rho)\\
=&x+\frac{\rho}{k}\sum_{i=1}^{k}
\sign(W_{x^{i-1},n_x}-W_{x^{i-1},y}+W_{R,n_x}-W_{R,y})\\
=&x+\rho    \sign(W_{R,n_x}-W_{R,y}).\\
\end{array}
\end{equation*}
Since $V=W_{R,n_x}-W_{R,y}$ is a random variable in $[-2\lambda,2\lambda]$,
$\sign(V)$ is a random variable in $\{-1,1\}^n$.
By Lemma \ref{lm-cha},
\begin{equation*}
%%\small
\renewcommand{\arraystretch}{1.6}
%\scriptsize
\begin{array}{ll}
&\Ca(B_{\widetilde{\F}},\A_3,\U_{m,n}(\lambda))\\
=&\E_{x\sim \D_{\O}}\E_{W_R\sim \U_{m,n}(\lambda)}
[\ID(\widehat{B}_{\F}(\A_3(x,\widetilde{\F}))\ne \widehat{B}_{\F}(x))]\\
\le& \E_{x\sim \D_{\O}}\E_{W_R\sim \U_{m,n}(\lambda)}[
\ID(||W_{R,n_x}-W_{R,y}||_{-\infty}\le \mu)+\\
&\ID(||W_{R,n_x}-W_{R,y}||_{-\infty}>\mu)[\ID(\widehat{B}_{\F}(\A_3(x,\widetilde{\F}))\ne \widehat{B}_{\F}(x))]\\
\le& (1-2^n(1-T(\lambda ,\mu))^n)+\\
&
\E_{x\sim \D_{\O}}
\sum_{V \in\{-1,1\}^n}\ID(\widehat{B}_{\F}(x+\rho V)\ne \widehat{B}_{\F}(x))]\\
%
%\le (1-2^n(1-T(\lambda ,a))^n)+2^n(1-T(\lambda ,a))^nE_{x\sim \D_{\O}}[\sum_{U \in\{-1,1\}^n}Q(x,U )/2^n]\\
\le& (1-2^n(1-T(\lambda ,\mu))^n)+\Ca(\F,\rho)\\

\end{array}
\end{equation*}
where $T(\lambda ,\mu)=1-\frac{(2\lambda -\mu)^2}{8\lambda ^2}$.  We have
$2^n(1-T(\lambda ,\mu))^n
=(2-2+\frac{4\lambda ^2+\mu^2-4\lambda \mu}{4\lambda ^2})^n
=(1-\frac{4\lambda \mu-\mu^2}{4\lambda ^2})^n
\ge 1-n\frac{4\lambda \mu-\mu^2}{4\lambda ^2}
\ge 1-n\mu/\lambda .$
So,
\begin{equation*}
\begin{array}{ll}
&\Ca(B_{\widetilde{\F}},\A_3,\U_{m,n}(\lambda))\\
\le& 1-2^n(1-T(\lambda, \mu))^n+\Ca(\F,\rho)\\

\le& \Ca(\F,\rho)+n\mu/\lambda .\\
\end{array}
\end{equation*}
The theorem is proved.
\end{proof}

\subsection*{Appendix D.  Proof of Theorem \ref{th-safe4}}

\gtthmP*
%{\em Proof of Theorem \ref{th-safe4}.}
\begin{proof}
Let $y\in\{0,1\}$ be the label of $x$.
Denote
$U=W_{x,1-y}-W_{x,y}\in\R^{1\times n}$
and  $Z=W_{R,1-y}-W_{R,y}\in\R^{1\times n}$.
We have
\begin{equation*}%\small
\begin{array}{ll}
&\sign(\frac{\nabla L(\F(x),y)}{\nabla x})\\
=&\sign(\frac{e^{\F_{1-y}(x)}(\frac{\nabla(\F_{1-y}(x))}{\nabla x}-\frac{\nabla (\F_y(x))}{\nabla x})}{e^{\F_y(x)}+e^{\F_{1-y}(x)}})\\
=&\sign(\frac{\nabla(\F_{1-y}(x))}{\nabla x}-\frac{\nabla (\F_y(x))}{\nabla x})\\
=&\sign(W_{x,1-y}-W_{x,y})\\
=&\sign(U).\\
%=C_{n}^{k}T(\lambda ,a)^{(n-k)}(1-T(\lambda ,a)^k)
\end{array}
\end{equation*}

From equation \eqref{eq-fg11}, we have
\begin{equation*}
\begin{array}{ll}
&\sign(\frac{\nabla L({\widetilde{\F}}(x),y)}{\nabla x})\\
%&=\frac{e^{{\F}_{1-y}(x)}}{\sum_{i=1}^m  e^{{\F}_{i}(x)}}
%(W_{x,1-y}-W_{x,y}),\\
%
%&\frac{\nabla L(\widetilde{\F}(x),y)}{\nabla x}\\
=&\sign(\frac{e^{\widetilde{\F}_{1-y}(x)}}{\sum_{i=1}^m  e^{\widetilde{\F}_{i}(x)}}
(W_{x,1-y}-W_{x,y} + W_{R,1-y}-W_{R,y})).\\
=&\sign(U+Z).\\
\end{array}
\end{equation*}
%Use $\alpha(U,Z)$ to denote the angle between two vectors $U$ and $Z$.
%Then
%\begin{equation*}%\small
%\begin{array}{ll}
%&\sign(\frac{\nabla L(\widetilde{\F}(x),y)}{\nabla x})\\
%%=&\sign(\frac{e^{\F_{R,1-y}(x)+\F_{1-y}(x)}\Delta_1}{e^{\F_{R,y}(x)+\F_y(x)}+e^{\F_{R,1-y}(x)+\F_{1-y}(x)}})\\
%%=&\sign(\frac{\nabla(\F_{R,1-y}(x)+\F_{1-y}(x))}{\nabla x}-\frac{\nabla (\F_{R,y}(x)+\F_y(x))}{\nabla x})\\
%%
%=&\sign(W_{x,1-y}-W_{x,y}+W_{R,1-y}-W_{R,y})\\
%=&\sign(U+Z).\\
%%=C_{n}^{k}T(\lambda ,a)^{(n-k)}(1-T(\lambda ,a)^k)
%\end{array}
%\end{equation*}

For $i\in[n]$, $\sign(U_i)=\sign(U_i+Z_i)$ if and only if
$(Z_i\le -U_i$ when $U_i\le0$) or
$(Z_i\ge-U_i$ when $U_i\ge0$),
where $Z_i, U_i$ are respectively the $i$-th coordinates of $Z, U$.
%Moveover, $Z=(W_{R,1-y})-(W_{R,y})$.
Since $W_R\sim \U_{m,n}(\lambda)$,  $Z=W_{R,1-y}-W_{R,y}$ is the difference of
two uniform distributions in $[-\lambda ,\lambda ]$.
By Lemma \ref{lm-cha},
     $U_i>0$ implies $P(Z_i\ge -U_i)=T(\lambda ,|U_i|)<T(\lambda ,\mu)$,
and $U_i<0$ implies $P(Z_i\le -U_i)=T(\lambda ,|U_i|)<T(\lambda ,\mu)$.
Hence, no matter what is the value of $U$,
we always have $P(\sign(U)=\sign(U+Z))<T(\lambda ,\mu)^n$,
where $T(\lambda ,\mu)=1-\frac{(2\lambda -\mu)^2}{8\lambda ^2}$.

Moreover, for $i\in[n]$, if $\sign(U_i)\ne\sign(U_i+Z_i)$,
we have ($Z_i>0$  when $U_i<0$) or ($Z_i<0$ when $U_i>0$).
So,
$P(\sign(U_i)\ne\sign(U_i+Z_i))<1/2<T(\lambda ,\mu)$,
since $T(\lambda ,\mu)$ is always $\ge1/2$.

Since $\{Z_i\}_{i\in[n]}$ is iid,
by Lemma \ref{lm-cha},  for a random vector $V \in\{-1,1\}^n$ we have
\begin{equation*}
\renewcommand{\arraystretch}{1.6}
%\tiny
%%\scriptsize
\begin{array}{ll}
&P_{W_R\sim \U_{m,n}(\lambda)}(\sign(\frac{\nabla L(\widetilde{\F}(x),y)}{\nabla x})=V )\\
%=P_{W_R\sim \U_{m,n}(\lambda)}(\sign\frac{e^{\F_{R,1-y}(x)+\F_{1-y}(x)}(\frac{\nabla(\F_{R,1-y}(x)+\F_{1-y}(x))}{\nabla x}-\frac{\nabla (\F_{R,y}(x)+\F_y(x))}{\nabla x})}{e^{\F_{R,y}(x)+\F_y(x)}+e^{\F_{R,1-y}(x)+\F_{1-y}(x)}})=V)\\
%=P_{W_R\sim \U_{m,n}(\lambda)}(\sign(\frac{e^{\F_{R,1-y}(x)+\F_{1-y}(x)}(U+V)}{e^{\F_{R,y}(x)+\F_y(x)}+e^{\F_{R,1-y}(x)+\F_{1-y}(x)}})=V)\\
%=P_{W_R\sim \U_{m,n}(\lambda)}(\sign(\frac{\nabla L(\widetilde{\F}(x),y)}{\nabla x})=V )\\
=&P_{W_R\sim \U_{m,n}(\lambda)}(\sign(U+Z)=V)\\
=&\prod_{i=1}^{n}P_{W_R\sim \U_{m,n}(\lambda)}(\sign(U_i+Z_i)=V_i)\\
=&\prod_{i=1}^{n}(\ID(\sign(U_i)=V_i)
   P_{W_R\sim \U_{m,n}(\lambda)}(\sign(U_i)=\sign(U_i+Z_i))\\
   &+\ID(\sign(U_i)\ne V_i)
   P_{W_R\sim \U_{m,n}(\lambda)}(\sign(U_i)\ne\sign(U_i+Z_i)))\\
\le& \prod_{i=1}^{n}(\ID(\sign(U_i)=V_i)T(\lambda ,\mu)+\ID(\sign(U_i)\ne V_i)T(\lambda ,\mu))\\
= &T(\lambda ,\mu)^{n}.\\
\end{array}
\end{equation*}

For $V\in\{-1,1\}^n$, denote
 $Q(x,V,\rho)=\ID(\widehat{B}_{\F}(x+\rho V)\ne\widehat{B}_{\F}(x))$.
We have
\begin{equation*}
\renewcommand{\arraystretch}{1.5}
%\tiny
\begin{array}{ll}
&\Ca(B_{\widetilde{\F}},\A_2,\U_{m,n}(\lambda))\\
=&\E_{x\sim D_{\O}}\E_{W_R\sim \U_{m,n}(\lambda)} [\ID(\widehat{B}_{\F}(x+\rho  \sign(\frac{\nabla L(\widetilde{\F}(x),y)}{\nabla x}))\ne\widehat{B}_{\F}(x))]]\\
=&\E_{x\sim D_{\O}}[\sum_{V\in\{-1,1\}^n}\\
&P_{W_R\sim \U_{m,n}(\lambda)}(\sign(\frac{\nabla L(\widetilde{\F}(x),y)}{\nabla x})=V)Q(x,V,\rho)]\\
\le& (T(\lambda ,\mu))^{n} \E_{x\sim D_{\O}}[(\sum_{V\in\{-1,1\}^n}Q(x,V,\rho))]\\
%\le \E_{x\sim D_{\O}}[(\sum_{V\in\{-1,1\}^n}\frac{1}{2^n}Q(x,V,r))I(\Ca(x,r)<\epsilon)]*(2T(\lambda ,\mu))^n+P_{x\sim D_{\O}}(\Ca(x,r)>\epsilon)\\
%\le (T(\lambda ,\mu)/2)^n\E_{x\sim \D_{\O}}[\sum_{U \in\{-1,1\}^n}\frac{1}{2^n}Q(x,j)]\\
\le& (2T(\lambda ,\mu))^n\Ca(\F,\rho)
\end{array}
\end{equation*}
where $T(\lambda ,\mu)=1-\frac{(2\lambda -\mu )^2}{8\lambda ^2}$.
We have
$(2T(\lambda ,\mu))^n
=(2-\frac{4\lambda ^2+\mu^2-4\lambda \mu}{4\lambda ^2})^n
=(1+\frac{4\lambda \mu-\mu^2}{4\lambda ^2})^n
\le (1+\frac{\mu}{\lambda })^n
\le e^{n\mu/\lambda }$.
%\le e^{\ln 2}
%=2$.
%
Hence,  $\Ca(B_{\widetilde{\F}},\A_2,\U_{m,n}(\lambda))<e^{n\mu/\lambda }\Ca(\F,\rho)$.
The theorem is proved.
%\qed
\end{proof}

\subsection*{Appendix E.  Proof of Theorem \ref{th-safe5}}

\gtthmQ*
%{\em Proof of Theorem \ref{th-safe5}.}
\begin{proof}
The proof is similar to that of Theorem \ref{th-safe6}.
So certain details of the proof are omitted.
From equation \eqref{eq-fg11}, we have
\begin{equation*}
\begin{array}{l}
\frac{\nabla L(\widetilde{\F}(x),y)}{\nabla x}
=\frac{\sum_{i=1}^{m}(W_{R,i}+W_{x,i}-W_{R,y}-W_{x,y})e^{\widetilde{\F}_i(x)}}{\sum_{i=1}^{m}e^{\widetilde{\F}_i(x)}}.
\end{array}
\end{equation*}

Let $m_x =\arg\max_{i\ne y}\{\langle W_{R,i},x\rangle\}$ and consider two
conditions $C_1$ and $C_2$:

{\bf Conditions $C_1$}: $\langle W_{R,m_x },x\rangle\ >\ \langle W_{R,j},x\rangle+\sqrt{\lambda}$ for all $j\in[m]\setminus\{y,m_x\}$.

{\bf Conditions $C_2$}: $||W_{R,m_x}-W_{R,y}||_{-\infty}>\mu$.

Note that condition $C_2$ implies
$\sign((W_{R,i }-W_{R,y}+W_{x,i}-W_{x,y})
=\sign(W_{R,i}-W_{R,y})$.

We give the probabilities for conditions $C_1$ and $C_2$ to be valid.
From the proof of Theorem \ref{th-safe6},
\begin{equation*}
%\small
\renewcommand{\arraystretch}{1.4}
\begin{array}{ll}
&P_{W_R\sim \M_{m,n}(\lambda)}(C_1)
%\ge& P_{W_R\sim \M_{m,n}(\lambda)}\\
%&(|\langle W_{R,j},x \rangle-\langle W_{R,i},x \rangle|>\sqrt{\lambda},\ \forall i,j\in[m]/\{y\},\ i\ne j)\\
%\ge&\Pi_{i=0}^{m-2}(1-\frac{2i\sqrt{\lambda}}{2\lambda})\\
%\ge&(1-\frac{m-2}{\sqrt{\lambda}})^{m-2}\\
\ge1-\frac{(m-2)^2}{\sqrt{\lambda} }.
\end{array}
\end{equation*}

Let $f(x)$ be the density function of $W_{R,m_x }$. Then
\begin{equation*}
\begin{array}{ll}
&P_{W_R\sim \U_{m,n}(\lambda)}(C_2)\\
&\ge P_{W_R\sim \U_{m,n}(\lambda)}(||W_{R,i}-W_{R,y}||_{-\infty}>\mu, \forall i\ne y)\\
&\ge(1-\frac{(m-1)\mu}{\lambda})^{n}\\
&\ge1-\frac{(m-1)n\mu}{\lambda}.
\end{array}
\end{equation*}
For $V\in\{-1,1\}^n$, it is  also easy to see
\begin{equation*}
%\small
\renewcommand{\arraystretch}{1.5}
\begin{array}{ll}
&P(\sign(W_{R,m_x }-W_{R,y})=V,\ C_1,\ C_2)\\
\le& P(\sign(W_{R,m_x }-W_{R,y})=V)\\
=&\sum_{i<y}P(m_x=i,    \sign(W_{R,y})=V)+\\
&\sum_{i>y}P(m_x=i,   \sign(W_{R,i})=V)\\
\le& \sum_{i<y}P(\sign(W_{R,y})=V)+\\
&\sum_{i>y}P(\sign(W_{R,i})=V)\\
=&\frac{m-1}{2^n}.
\end{array}
\end{equation*}

If conditions $C_1$ and $C_2$ are satisfied, then for any $y\in[m]\setminus\{y,m_x\}$, we have
\begin{equation*}
\renewcommand{\arraystretch}{1.5}
\begin{array}{ll}
 &||W_{R,m_x }+W_{x,m_x }-W_{R,y}-W_{x,y}||_{-\infty} e^{\widetilde{\F}_{m_x }(x)}\\
 >&\mu/2e^{\widetilde{\F}_{m_x }(x)}\\
 >&\mu/2e^{{\widetilde{\F}_{j}(x)}+\sqrt{\lambda}-2\beta-n\mu/2}\\
=&\mu/2e^{-2b-n\mu/2}e^{\sqrt{\lambda}}e^{{\widetilde{\F}_{j}(x)}}\\
>&(2\lambda+\mu)me^{{\widetilde{\F}_{j}(x)}}\\
>&m||W_{R,j}+W_{x,j}-W_{R,y}-W_{x,y}||_\infty e^{\widetilde{\F}_{j}(x)}
\end{array}
\end{equation*}
which means
\begin{equation*}
%\small
\begin{array}{ll}
&\sign(\sum_{i=1}^{m}(W_{R,i}+W_{x,i}-W_{R,y}-W_{x,y})e^{\widetilde{\F}_i(x)})\\
=&\sign((W_{R,m_x }+W_{x,m_x }-W_{R,y}-W_{x,y})e^{\widetilde{\F}_{m_x }(x)}),
\end{array}
\end{equation*}
and hence
\begin{equation*}
%\small
\renewcommand{\arraystretch}{1.5}
\begin{array}{ll}
&\sign(\frac{\nabla L(\widetilde{\F}(x),y)}{\nabla x})\\
=&\sign(\frac{\sum_{i=1}^{m}(W_{R,i}+W_{x,i}-W_{R,y}-W_{x,y})e^{\widetilde{\F}_i(x)}}{\sum_{i=1}^{m}e^{\widetilde{\F}_i(x)}})\\
=&\sign(\sum_{i=1}^{m}(W_{R,i}+W_{x,i}-W_{R,y}-W_{x,y})e^{\widetilde{\F}_i(x)})\\
=&\sign((W_{R,m_x }+W_{x,m_x }-W_{R,y}-W_{x,y})e^{\widetilde{\F}_{m_x }(x)})\\
=&\sign(W_{R,m_x }+W_{x,m_x }-W_{R,y}-W_{x,y})\\
=&\sign(W_{R,m_x }-W_{R,y}).
\end{array}
\end{equation*}
Hence
\begin{equation*}
\renewcommand{\arraystretch}{1.5}
%\tiny
\begin{array}{ll}
&\E_{W_R\sim \U_{m,n}(\lambda)}\\
 &[\ID(\widehat{B}_{\F}(x+\rho\sign(\frac{\nabla L(\widetilde{\F}(x),y)}{\nabla x}))\ne\widehat{B}_{\F}(x))\ID(C_1,\  C_2)]\\
=&\E_{W_R\sim \U_{m,n}(\lambda)} [\ID(\widehat{B}_{\F}(x+\rho\sign(W_{R,m_x }-W_{R,y}))\ne\widehat{B}_{\F}(x))\\
&\ID(C_1,\ C_2)]\\
=&\sum_{V\in\{-1,1\}^{n}}P(\sign(W_{R,m_x }-W_{R,y})=V,\ C_1,\ C_2)\\
&\ID(\widehat{B}_{\F}(x+\rho V)\ne\widehat{B}_{\F}(x))\\
&\le \frac{m-1}{2^n} \sum_{V\in\{-1,1\}^{n}}\ID(\widehat{B}_{\F}(x+\rho V)\ne\widehat{B}_{\F}(x))\\
=&(m-1)\Ca(\F,\rho).
\end{array}
\end{equation*}
Finally, we have
\begin{equation*}
\renewcommand{\arraystretch}{1.5}
%\tiny
\begin{array}{ll}
&\Ca(B_{\widetilde{\F}}\A_2,\U_{m,n}(\lambda))\\
=&\E_{x\sim D_{\O}}\E_{W_R\sim \U_{m,n}(\lambda)} [\ID(\widehat{B}_{\F}(x+\rho  \sign(\frac{\nabla L(\widetilde{\F}(x),y)}{\nabla x}))\ne\widehat{B}_{\F}(x))]\\
\le&\E_{x\sim D_{\O}}\E_{W_R\sim \U_{m,n}(\lambda)} [\ID(\widehat{B}_{\F}(x+\rho  \sign(\frac{\nabla L(\widetilde{\F}(x),y)}{\nabla x}))\ne\widehat{B}_{\F}(x))\\
&\ID(C_1,\ C_2)+(1-\ID(C_1))+(1-\ID(C_2))]\\
\le& (m-1)\Ca(\F,\rho)+\E_{x\sim D_{\O}}\E_{W_R\sim \U_{m,n}(\lambda)} [(1-\ID(C_1))+(1-\ID(C_2))]\\
\le& (m-1)\Ca(\F,\rho)+\E_{x\sim D_{\O}}[1-P_{W_R\sim \U_{m,n}(\lambda)}(C_1)]+\\
&\E_{x\sim D_{\O}}[1-P_{W_R\sim \U_{m,n}(\lambda)}(C_2)]\\
%\le& (m-1)\Ca(\F,\rho)+\frac{(m-1)(m-2)\mu}{\lambda}+\frac{(m-2)^2}{\sqrt{\lambda} \eta}+P_{x\sim D_{\O}}(|x|_{\infty}\le \eta)\\
\le& (m-1)\Ca(\F,\rho)+\frac{(m-1)n\mu}{\lambda}+\frac{(m-2)^2}{\sqrt{\lambda}}.\\
\end{array}
\end{equation*}
The theorem is proved.
%\qed
\end{proof}

\subsection*{Appendix F. Structures of DNN models used in the experiments}
\label{sec-a-dnn}

The networks in section \ref{sec-exp1}:

%{\bf Networks for MNIST:}

{\bf Networks $\F^{(1)}$ and $\F^{(2)}$ for MNIST} have the same structure:

Input layer: $N\times 1\times  28\times  28$,   where $N$ is steps of training.

Hidden layer 1: a convolution layer with kernel $1\times 32\times 3\times 3$ with padding$=1$ $\to$ do a batch normalization $\to$ do Relu $\to$ use max pooling with step=2.

Hidden layer 2: a convolution layer  with kernel $32\times 64\times 3\times 3$ with padding$=1$ $\to$ do a batch normalization $\to$ do Relu $\to$ use max pooling with step=2.

Hidden layer 3: a convolution layer  with kernel $64\times 128\times 3\times 3$ with padding$=1$ $\to$ do a batch normalization $\to$ do Relu $\to$ use max pooling with step=2.

Hidden layer 4: draw the output as $N\times 128\times 3\times 3$ $\to$ use a full connection with output size $N\times 128\times 2$ $\to$ do Relu.

Hidden layer 4: use a full connection with output size $N\times 100$ $\to$ do Relu.

Output layer: a full connection layer with output size $N\times 10$.

{\bf Networks $\F^{(1)}$ and $\F^{(2)}$ for CIFAR-10} have the same structure:

Input layer: $N\times 3\times  32\times  32$,   where $N$ is steps of training.

Hidden layer 1: a convolution layer with kernel $3\times 64\times 3\times 3$ with padding$=1$ $\to$ do a batch normalization $\to$ do Relu.

Hidden layer 2: a convolution layer  with kernel $64\times 64\times 3\times 3$ with padding$=1$ $\to$ do a batch normalization $\to$ do Relu.

Hidden layer 3: a convolution layer  with kernel $64\times 128\times 3\times 3$ with padding$=1$ $\to$ do a batch normalization $\to$ do Relu.

Hidden layer 4: a convolution layer  with kernel $128\times 128\times 3\times 3$ with padding$=1$ $\to$ do a batch normalization $\to$ do Relu $\to$ use max pooling with step=2.

Hidden layer 5: a convolution layer  with kernel $128\times 256\times 3\times 3$ with padding$=1$ $\to$ do a batch normalization $\to$ do Relu.

Hidden layer 6: a convolution layer  with kernel $256\times 256\times 3\times 3$ with padding$=1$ $\to$ do a batch normalization $\to$ do Relu.

Hidden layer 7: a convolution layer  with kernel $256\times 256\times 3\times 3$ with padding$=1$ $\to$ do a batch normalization $\to$ do Relu $\to$ use max pooling with step=2.

Hidden layer 8: a convolution layer  with kernel $256\times 512\times 3\times 3$ with padding$=1$ $\to$ do a batch normalization $\to$ do Relu.

Hidden layer 9: a convolution layer  with kernel $512\times 512\times 3\times 3$ with padding$=1$ $\to$ do a batch normalization $\to$ do Relu.

Hidden layer 10: a convolution layer  with kernel $512\times 512\times 3\times 3$ with padding$=1$ $\to$ do a batch normalization $\to$ do Relu $\to$ use max pooling with step=2.

Hidden layer 11: a convolution layer  with kernel $512\times 512\times 3\times 3$ with padding$=1$ $\to$ do a batch normalization $\to$ do Relu.

Hidden layer 12: a convolution layer  with kernel $512\times 512\times 3\times 3$ with padding$=1$ $\to$ do a batch normalization $\to$ do Relu.

Hidden layer 13: a convolution layer  with kernel $512\times 512\times 3\times 3$ with padding$=1$ $\to$ do a batch normalization $\to$ do Relu $\to$ use max pooling with step=2.

Hidden layer 14: draw the output as $N\times 2048$ $\to$ a full connection layer with output size $N\times 1024$ $\to$ do Relu.

Hidden layer 15: a full connection layer with output size $N\times 512$ $\to$ do Relu.

Hidden layer 16: a full connection layer with output size $N\times 128$ $\to$ do Relu.

Output layer: a full connection layer with output size $N\times 10$.

\end{document}